\documentclass[11pt]{article}

\usepackage{amsfonts, amsmath,amssymb,array,latexsym, hyperref, amscd}

\usepackage{fancyhdr,a4wide}
\usepackage{amsthm}
\usepackage{graphicx}
\usepackage{xcolor}
\usepackage{enumitem}

\newtheorem{theorem}{Theorem}[section]
\newtheorem{lemma}[theorem]{Lemma}
\newtheorem{corollary}[theorem]{Corollary}
\newtheorem{proposition}[theorem]{Proposition}

\numberwithin{equation}{section}

\newcommand{\twiddle}[1]{\widetilde{#1}}

\newcommand{\norm}[1]{\left|\left|#1\right|\right|}
\newcommand{\lr}[1]{\left(#1\right)}
\newcommand{\abs}[1]{\left|#1\right|}
\newcommand{\set}[1]{\left\{#1\right\}}
\newcommand{\E}[1]{\mathbb E\left[#1\right]}
\newcommand{\inprod}[2]{\left \langle #1,#2\right\rangle }
\def\Var{\mathrm{Var}}

\newcommand{\R}{\mathbb R}
\newcommand{\C}{\mathbb C}
\newcommand{\gives}{\rightarrow}

\newcommand{\mN}{\mathcal N}

\newcommand{\N}{\mathbb N}

\newcommand{\sgn}[1]{\mathrm{sgn}\left(#1\right)}
\newcommand{\tr}{\mathrm{Tr}}
\newcommand{\G}[5]{\ensuremath{G_{#2}^{#1}\lr{ #3 ~\bigg| \Gparams{#4}{#5} }}}
\newcommand{\Gparams}[2]{\begin{array}{c}
     #1  \\
     #2 
\end{array}}

\newcommand{\w}{\omega}
\newcommand{\im}{\mathrm{im}}

\newcommand{\lpost}{\lambda_{\mathrm{post}}}
\newcommand{\lpre}{\lambda_{\mathrm{prior}}}

\newcommand{\Paprior}{\mathbb P_{\mathrm{prior}}}
\newcommand{\Papost}{\mathbb P_{\mathrm{post}}}

\DeclareMathOperator*{\argmax}{arg\,max} 
\DeclareMathOperator*{\argmin}{arg\,min} 
\newcommand{\col}{\mathrm{col}}
\newcommand{\Ep}[1]{\mathbb E_{\mathrm{prior}}\left[#1\right]}

\author{Boris Hanin\footnote{BH is supported by NSF grants DMS-2143754, DMS-1855684, and DMS-2133806},  Alexander Zlokapa\footnote{AZ is supported by the Hertz Foundation, and by the Department of Defense through the National Defense Science and Engineering Graduate Fellowship Program}\\
$\,\textsuperscript{*}$ Princeton ORFE\\
$\,\textsuperscript{$\dagger$}$ MIT Center for Theoretical Physics, Google Quantum AI}

\title{Bayesian Interpolation with Deep Linear Networks}
\begin{document}
\maketitle
%\tableofcontents

\begin{abstract}
Characterizing how neural network depth, width, and dataset size jointly impact model quality is a central problem in deep learning theory. We give here a complete solution in the special case of linear networks with output dimension one trained using zero noise Bayesian inference with Gaussian weight priors and mean squared error as a negative log-likelihood. For any training dataset, network depth, and hidden layer widths, we find non-asymptotic expressions for the predictive posterior and Bayesian model evidence in terms of Meijer-G functions, a class of meromorphic special functions of a single complex variable. Through novel asymptotic expansions of these Meijer-G functions, a rich new picture of the joint role of depth, width, and dataset size emerges. We show that linear networks make provably optimal predictions at infinite depth: the posterior of infinitely deep linear networks with data-agnostic priors is the same as that of shallow networks with evidence-maximizing data-dependent priors. This yields a principled reason to prefer deeper networks when priors are forced to be data-agnostic. Moreover, we show that with data-agnostic priors, Bayesian model evidence in wide linear networks is maximized at infinite depth, elucidating the salutary role of increased depth for model selection. Underpinning our results is a novel emergent notion of effective depth, given by the number of hidden layers times the number of data points divided by the network width; this determines the structure of the posterior in the large-data limit.

\end{abstract}

\section{Introduction}
%\subsection{Significance Statement}

%Understanding the interplay between network architecture, dataset statistics, and learning algorithms is a key challenge in deep learning. We overcome this challenge analytically for zero-noise Bayesian inference in neural networks with linear activations and output dimension one but arbitrary depth, width, and training data. Our results provide the first rigorous solutions to learning with neural networks in which depth, width, and dataset can vary freely. We identify an emergent notion of effective depth, which combines depth, width, and dataset size, that must be large for networks to compute optimal posteriors from universal priors. Our results both present novel examples of models admitting exact solutions to Bayesian inference and give the first principled Bayesian justification for preferring deeper networks.

\subsection{Background}

A central aim of deep learning theory is to understand the properties of overparameterized networks trained to fit large datasets. Key questions include: how do learned networks use training data to make predictions on test points? Which neural network architectures lead to more parsimonious models? What are the joint scaling laws connecting the quality of learned models to the number of training data points, network depth and network width \cite{geiger2020scaling,kaplan2020scaling,bahri2021explaining,rae2021scaling}?

The present article gives the first exact answers to such questions for a class of neural networks in which one can simultaneously vary input dimension, number of training data points, network width, and network depth. This is significant because the limits where these four structural parameters tend to infinity do not commute, causing all prior work to miss important aspects of how they jointly influence learning. Our results pertain specifically to \textit{deep linear networks} 
\begin{align}
\label{eq:f-def}f(x) = W^{(L+1)}\cdots W^{(1)} x
\end{align}
with input dimension $N_0$, $L$ hidden layers of widths $N_\ell$, and output dimension $N_{L+1}=1$. As a form of learning we take zero noise Bayesian interpolation starting from a Gaussian prior on network weights $W^{(\ell)}\in \R^{N_{\ell}\times N_{\ell-1}}$ and empirical mean squared error over a training dataset of $P$ examples as the negative log-likelihood. Deep linear networks, while linear in $x$, are non-linear in their parameters and have been extensively studied as models for learning with neural networks using both gradient-based methods \cite{arora2018convergence,arora2018optimization,saxe2013exact,saxe2022neural,kawaguchi2016deep} and Bayesian inference \cite{li2021statistical,zavatone2021exact,zavatone2022contrasting}.

Since we are considering an output dimension of 1, we may write $f(x)=\theta^Tx$ for a vector $\theta\in \R^{N_0}$. What differentiates our work from a classical Bayesian analysis of Gaussian linear regression is that as soon as $L\geq 1$ the components of $\theta$ are correlated and non-Gaussian under the prior. Predictions $f(x)$ on inputs $x$ orthogonal to inputs from the training data therefore differ under the prior and posterior. Specifically, as shown in Figure \ref{fig:perp}, we may decompose $\theta=\theta_{||}+\theta_\perp$ into its projections onto directions spanned by the training data and their complement. By our Bayesian construction, $\theta_{||}$ is responsible for fitting the training data. Due to the correlations under the prior between $\theta_{||}$ and $\theta_{\perp}$, however, information from the training data will influence the posterior distribution of $\theta_{\perp}$. It is precisely this data-dependent extrapolation displayed by deep linear networks with $L\geq 1$ that differentiates them from the linear models obtained by taking $L=0$.

\subsection{Relation to Prior Work}
To put our work in context, we briefly summarize prior approaches to understanding learning with neural networks. The neural tangent kernel (NTK) and other kernel-based models \cite{du2017gradient,du2018gradient,allen2019learning, jacot2018neural,liu2022loss} reduce neural networks to linear models. Training by gradient descent or Bayesian inference consequently does not affect predictions $f(x_\perp)$ of test inputs orthogonal to the training dataset. Moreover, the NTK regime only considers the limit of infinite width with finite depth and dataset size. More recent work \cite{hanin2018neural,hanin2020products,hanin2018start, hanin2019finite, li2022neural,roberts2022principles, yaida2019non} shows that feature learning emerges when taking depth and width to infinity simultaneously with the effective prior depth $\lpre=L/N$ (see \eqref{eq:lpre-def}) determining both the behavior under both gradient descent and Bayesian inference. Still, such analyses are restricted to finite dataset size $P$ (or more precisely dataset sizes that are much smaller than both network depth and width).

Phenomena such as double descent \cite{belkin2019reconciling, belkin2020two} and benign overfitting \cite{bartlett2020benign} appear in linear models when taking width, dataset size, and dataset dimension to infinity simultaneously \cite{hastie2022surprises,montanari2022interpolation,adlam2019random,adlam2020neural,advani2020high,mei2019generalizationb, mei2021generalization}. However, as mentioned previously, these linear models do not learn to make data-adaptive predictions in directions not already present in the training data. Moreover, such approaches are restricted to studying fixed depth. Other approaches consider neural networks in the mean-field limit \cite{mei2018mean,rotskoff2018parameters,chizat2018global,sirignano2020mean,sirignano2021mean,yu2020normalization,pham2021global,yang2021tensoriv}. In this regime, networks do make data-adaptive predictions $f(x_\perp)$. However, mean-field limits have only been considered at fixed depth and recast optimization in terms of complex non-linear evolution equations, whose dependence on the training data is typically difficult to access.

The literature most directly related to the present article studies Bayesian inference with either non-linear \cite{lee2017deep, ariosto2022statistical, hron2022wide,cui2023optimal, naveh2021self,seroussi2021separation} or linear networks \cite{zavatone2021exact,noci2021precise,zavatone2022contrasting, li2021statistical}. These works consider only the regime where depth is either fixed or much smaller than both width and size of the training dataset. We find, in contrast, that the full role of depth in model selection and extrapolation can only be understood in the regime where depth, width, and dataset size are simultaneously large. Finally, our results are the first to characterize the behavior of deep neural networks at any joint scaling of depth, width, dataset size, and dataset dimension. In this sense, they can be viewed as giving exact expressions for the predictive posterior over deep Gaussian processes \cite{damianou2013deep} with Euclidean covariance in every layer. 

\subsection{Overview of Results}
Our first result, Theorem \ref{thm:Z-form} below, gives exact non-asymptotic formulas for both the predictive posterior (i.e., the distribution of $f(x)$ jointly for all inputs $x$ when $W^{(\ell)}$ are drawn from the posterior) and the Bayesian model evidence in terms of a class of meromorphic special functions in one complex variable called Meijer-G functions \cite{meijer1936whittakersche}. These results hold for arbitrary training datasets, input dimensions, hidden layer widths, and network depth. They represent a novel enlargement of the class of priors over $\theta$ for which posteriors can be computed in closed form. In particular, they show that zero noise Bayesian inference is exactly solvable for deep Gaussian processes \cite{damianou2013deep} with Euclidean covariances.

To glean insights from the non-asymptotic results in Theorem \ref{thm:Z-form}, we provide in Theorem \ref{thm:logG} new asymptotic expansions of Meijer-G functions that allow us to compute expressions for the Bayesian model evidence and the predictive posterior under essentially any joint scalings of $P,N_\ell,L$ in the large-data limit where $P\gives\infty$ with $P/N_0\gives \alpha_0\in (0,1)$. 
% \footnote{As explained in \S \ref{sec:limit-thms}, since our networks are linear with respect to the inputs and we require the trained model to interpolate the training data, the Bayesian posterior is a delta function on the minimal $\ell_2$-norm data interpolant as soon as $P\geq N_0.$}
To understand the role of depth, we consider regimes in which $L$ either stays finite or  grows together with $P,N_\ell$.

We focus in this article on zero-noise, or interpolating, posteriors that fit the training data exactly (see \eqref{eq:post-def} and the discussion in Optimal Extrapolation). For parametric models interpolation often causes overfitting at large sample sizes. An important empirical \cite{zhang2016understanding} and theoretical \cite{bartlett2020benign,belkin2019reconciling, hastie2022surprises} observation, however, is that in many non-parameteric overparameterized models, such as the deep linear neural networks we study, this does not occur. Our results therefore give new information about the joint effects of depth, width, and sample size on the nature of interpolating models.

What emerges from our analysis is a rich new picture of the role of depth in  linear networks in determining the nature of extrapolation and Bayesian model selection, given by maximizing the Bayesian model evidence, i.e., the likelihood of the training data under the posterior (cf \S 4, \S 5 in \cite{mackay1992bayesian}). We present here an informal explanation of our main results, starting with Theorem \ref{thm:bayesfeature} which implies the following:
\begin{quote}
\textbf{Takeaway:} Evidence-maximizing priors give the same Gaussian predictive posterior for any architecture in the large-data limit.
\end{quote}
This distinguished posterior represents, from a Bayesian point of view, a notion of optimal extrapolation. Indeed, because $f(x)$ is linear in $x$, it is natural to decompose 
\[
x = x_{||}+x_{\perp},\qquad x\in \R^{N_0},
\]
where $x_{||}$ is the projection of $x$ onto directions spanned by inputs from the training data, and $x_\perp$ is the projection of $x$ onto the orthogonal complement (like the decomposition of $\theta$ in Figure \ref{fig:perp}).

\begin{figure}
    \centering
    \includegraphics[scale=0.8]{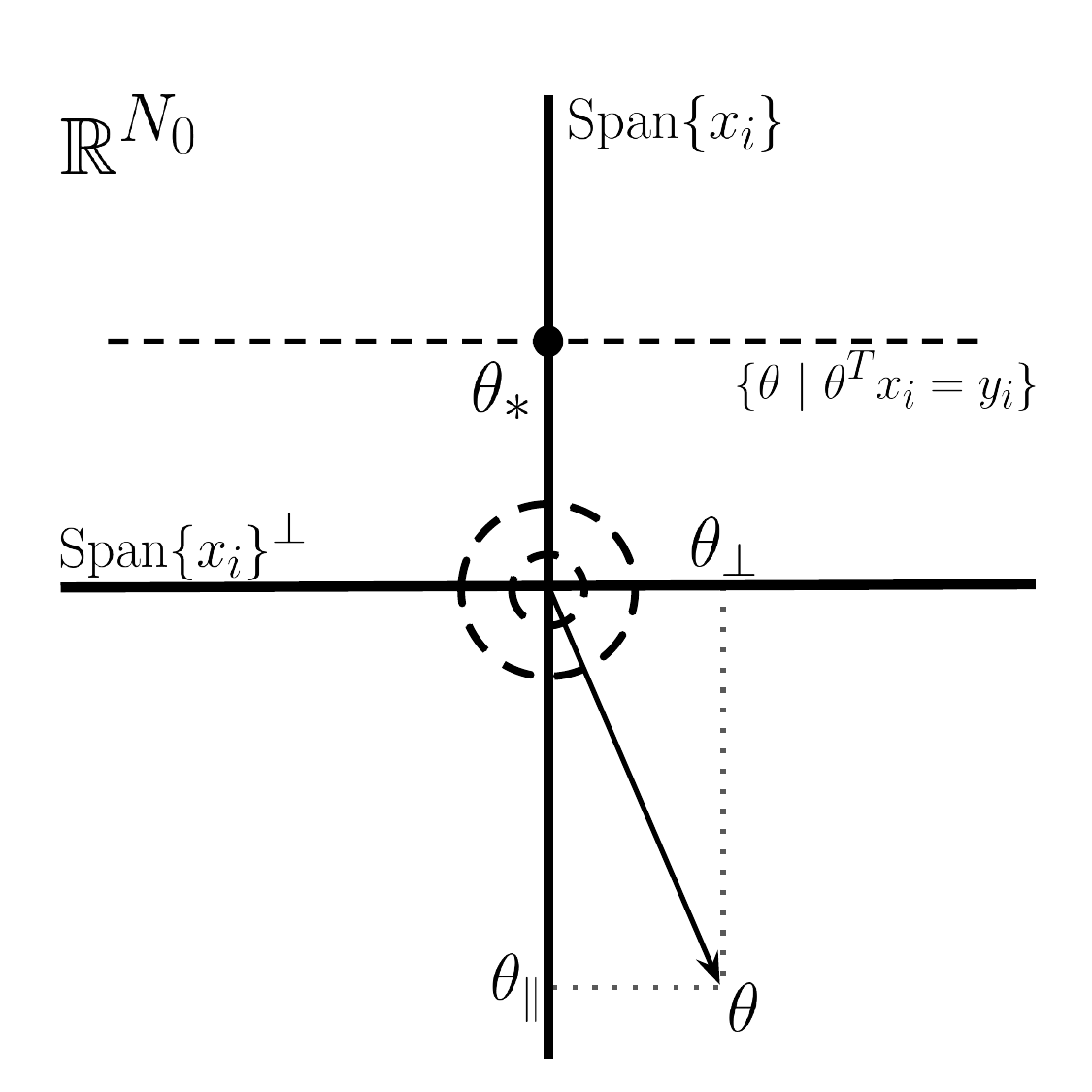}
    \caption{The input space $\R^{N_0}$ is decomposed into directions $\mathrm{Span}\set{x_i}$ spanned by inputs from the training data and its orthogonal complement. The minimal $\ell_2$ norm interpolant is the intersection between the space of interpolants (i.e., the dashed line showing all $\theta$ satisfying $\theta^Tx_i=y_i$ for all $i$) and $\mathrm{Span}\set{x_i}$. When generating predictions $f(x)=\theta^Tx$, the parameter vector $\theta$ can be decomposed into its projections $\theta_{||}$ and $\theta_\perp$ onto $\mathrm{Span}\set{x_i}$  and $\mathrm{Span}\set{x_i}^\perp$ respectively. When fully trained, $\theta_{||}$ will equal $\theta_*$. Circles centered at the origin represent equi-probable lines for the prior over $\theta$, which is always radial but non-Gaussian for any $L\geq 1$.}
    \label{fig:perp}
\end{figure}
Predictions under the evidence-maximizing posterior at test input $x$ are Gaussian and take the form
\begin{equation}\label{eq:opt-post-form}
f(x) \sim \mathcal N\lr{\theta_*^T x_{||},~ \frac{\norm{\theta_*}^2}{\alpha_0} \norm{x_\perp}^2}, 
\end{equation}
where $\theta_*$ is the minimal norm interpolant of the training data (cf. Figure \ref{fig:perp} and \eqref{eq:Vstar-def}). As we explain in Optimal Extrapolation below, the deterministic mean $\theta_*^Tx_{||}$ appears because, by construction, our posteriors are concentrated on $\theta$ that interpolate the training data (see \eqref{eq:post-def}) and thus we must have $\theta_{||}=\theta_*$. The scalar $\norm{\theta_*}^2/\alpha_0$, in contrast, sets a data-dependent variance for making predictions in directions orthogonal to inputs in the training data. This particular value for the variance is the most likely given that our posteriors are necessarily isotropic in directions orthogonal to the training data. Moreover, the approximate normality of the evidence-maximizing posterior at large sample sizes is due to the fact that the coordinates of a spherically symmetric random vector in high dimensions are approximately Gaussian (see the discussion just below \eqref{eq:post-form}).

In general, a linear network that maximizes Bayesian evidence, and hence produces the posterior \eqref{eq:opt-post-form}, may require the prior distribution over network weights to be data-dependent. In machine learning contexts, we hope instead to find optimal but data-\emph{agnostic} priors. Theorems \ref{thm:LN} and \ref{thm:PLN} (in combination with Theorem \ref{thm:bayesfeature}) show this is possible in linear networks with large depth. Informally they give:
\begin{quote}
\textbf{Takeaway:} Wide linear networks (at comparable depth, width, number of data points) with data-\emph{agnostic} priors give the same predictive posterior as shallow networks with optimal data-\emph{dependent} priors.
\end{quote}
This result highlights the remarkable role of depth in shaping the posterior over predictions in directions of feature space orthogonal to those present in the training data. This can only happen in non-linear models such as deep linear networks. Quantifying how large network depth must be to ensure optimal extrapolation is explained in Theorem \ref{thm:PLN}, which provides universal scaling laws for Bayesian posteriors in terms of a single parameter that couples depth, width, and dataset size. Informally, we have the following:
\begin{quote}
\textbf{Takeaway:} Consider linear networks in the regime $1\ll$ depth, dataset size $\ll$ width. With data-agnostic priors, the posterior depends only on the effective posterior depth
\begin{align*}
\lpost:= \frac{\text{(network depth)} \times \text{(dataset size)}}{ \text{network width}}.
\end{align*}
As $\lpost \gives \infty$, evidence grows and the posterior converges to the evidence-maximizing posterior \eqref{eq:opt-post-form}.
\end{quote}
Since $\lpost$ determines both the bias and the variance of the posterior, the preceding takeaway can be viewed as a scaling law relating depth, width, and training set size \cite{rae2021scaling,geiger2020scaling,kaplan2020scaling,bahri2021explaining}. In particular, it shows that for large linear networks it is $\lpost$, rather than depth, width, and dataset size separately, that determines the quality of the learned model.

% We refer the reader to \S \ref{sec:prior} for an discussion of why $\lpost$ is a natural measure of depth or complexity for the posterior.  We also note that $\lpost$ appears at leading order in the perturbative expressions (32) - (35) of \cite{zavatone2022contrasting} (in their notation $\lpost=\ell \gamma/\alpha$) for computing the $\ell_2$-generalization error in the regime of fixed $L$ and $N_0,N_\ell,P\gives \infty$ with $P/N_\ell$ remaining order $1$. 

As the preceding statement suggests, at least with data-agnostic priors and wide linear networks, maximizing Bayesian evidence requires large depth, as measured by $\lpost$. Moreover, evidence maximization is not possible at finite $\lpost$. The final result we present (see Theorem \ref{thm:LN}) concerns maximization of Bayesian evidence --- a principled method of comparing different architectures \cite{mackay1992bayesian} --- and is summarized as follows:
\begin{quote}
\textbf{Takeaway:} With data-agnostic priors and width that is proportional to dataset size, Bayesian evidence is maximal in networks with depth equal to a data-dependent constant times width. 
\end{quote}
Mis-specification of this constant only results in an order one decrease in evidence and does not affect the posterior. In comparison, a network with smaller depth has exponentially smaller evidence and a suboptimal posterior. The preceding takeaways give perhaps the first principled Bayesian justification for preferring neural networks with large depth, albeit in the restricted setting of linear networks.

We then state our first result, Theorem \ref{thm:Z-form}, in \S \ref{sec:Z-form}. This gives an exact non-asymptotic formula for the characteristic function of predictive posterior and the model evidence in terms of Meijer-G functions. We complement this in \S \ref{sec:G-results} by giving in Theorem \ref{thm:logG} asymptotic expansions for Meijer-G functions. The next collection of results, provided in \S \ref{sec:limit-thms}, details the model evidence and posterior as number of datapoints, input dimension, width, and depth tend to infinity in various regimes. The results in \S \ref{sec:Lfinite-thms} pertain specifically to the analysis of networks with a finite number of hidden layers in the regime where number of training datapoints, input dimension, and width tend to infinity. The main result is Theorem \ref{thm:finite-L}. In contrast, \S \ref{sec:Linfinit-thms} and \S \ref{sec:scaling-thms} consider regimes in which depth also tends to infinity. The main result is Theorem \ref{thm:PLN}. Finally, \S \ref{sec:properties} contains simple corollaries connecting our results to scaling laws for the generalization error (Theorem \ref{thm:scaling}) and double descent (Theorem \ref{thm:dd}).

\section{Preliminaries}
\subsection{Setup}
We fix a training set with $P$ examples
\[
X_{N_0}=\lr{x_{1,N_0},\ldots,x_{P,N_0} }\in \R^{N_0 \times P},\; Y_{N_0}=\lr{y_1,\ldots, y_P}\in \R^P.
\]
We will assume $X_{N_0}$ has full rank. %and consider a linear neural network with $L$ hidden layers of width $N_\ell$ and an output layer of dimension $N_{L+1}=1$ as in \eqref{eq:f-def}. 
Since we study zero noise posteriors supported on the models that minimize the likelihood \eqref{eq:L-def}, we assume also that $1\leq P \leq N_0$. Otherwise, the set of minima of the likelihood consists of a single $\theta$ and our posteriors would have zero variance. A key role in our results will be played by the minimal $\ell_2$-norm solution to ordinary linear least squares regression of $Y_{N_0}$ onto $X_{N_0}$:
\begin{equation}\label{eq:min-theta}
\theta_{*,N_0}:=\argmin_{\theta\in \R^{N_0}}\norm{\theta}_2\quad \text{s.t.} \quad \theta^T X_{N_0} = Y_{N_0}.
\end{equation}
Further, we fix $N_1,\ldots, N_L\geq 1$ and consider fitting the training data $(X_{N_0},Y_{N_0})$ by a linear model
\[
f(x) = \theta^Tx\in \R,\quad \theta,x\in \R^{N_0}
\]
equipped with quadratic negative log-likelihood
\begin{equation}\label{eq:L-def}
\mathcal L(\theta~|~X_{N_0},Y_{N_0}):=\frac{1}{2}\norm{\theta^TX_{N_0}-Y_{N_0}}_2^2
\end{equation}
and a \textit{deep linear prior}
\begin{equation}\label{eq:prior-def}
\theta\sim \Paprior \quad \Longleftrightarrow \quad \theta = W^{(L+1)}\cdots W^{(1)}    
\end{equation}
in which 
\[
W^{(\ell)}\in \R^{N_{\ell}\times N_{\ell-1}},\quad W_{ij}^{(\ell)}\sim \mN\lr{0,\frac{\sigma^2}{N_{\ell-1}}}\quad \text{independent},
\]
where $\sigma>0$. Our goal is to study the posterior distribution over the set of $\theta\in\R^{N_0}$ that exactly fit the training data. Explicitly, writing $d\Paprior(\theta)$ for the prior density, we study zero noise posteriors
\begin{align}
\label{eq:post-def}
&d\Papost\lr{\theta~|~L,N_\ell, \sigma^2, X_{N_0}, Y_{N_0}}\\
\nonumber&:= \lim_{\beta\gives \infty}\frac{d\Paprior\lr{\theta~|~N_0,L,N_\ell, \sigma^2}\exp\left[-\beta\mathcal L(\theta~|~X_{N_0},Y_{N_0})\right]}{Z_\beta\lr{X_{N_0},Y_{N_0}~|~L,N_\ell, \sigma^2}}.
\end{align}
Writing $\mathbb E_{\mathrm{post}}[\cdot]$ for the expectation with respect to the posterior \eqref{eq:post-def}, we describe the posterior by giving exact formulas for its characteristic function 
\begin{equation}\label{eq:char-part}
\mathbb E_{\mathrm{post}}\left[\exp\left\{-i{\bf t}\cdot \theta\right\}\right]=\frac{Z_\infty({\bf t})}{Z_\infty({\bf 0)}},\qquad {\bf t},\theta \in \R^{N_0},
\end{equation}
where $Z_\infty({\bf t})=Z_\infty( {\bf t}~|~L,N_\ell, \sigma^2,X_{N_0},Y_{N_0})$ is the zero-temperature partition function given by taking $\beta\gives \infty$ in
\begin{align}\label{eq:part-def}
 &Z_\beta({\bf t}) := A_\beta \int  \exp\Bigg[-\frac{\beta}{2}\big|\big|Y-\prod_{\ell=1}^{L+1} W^{(\ell)}X\big|\big|_2^2-i\theta\cdot {\bf t} -\sum_{\ell=1}^{L+1}\frac{N_{\ell-1}}{2\sigma^2}\big|\big|W^{(\ell)}\big|\big|_F^2  \Bigg]\prod_{\ell=1}^{L+1}dW^{(\ell)}.
\end{align}
The normalizing constant $A_\beta$ cancels in the ratio \eqref{eq:post-def} and in any computations involving maximizing ratios of model evidence (see \S \ref{sec:setup}).

The denominator $Z_\infty({\bf 0})$ is often called the Bayesian model evidence and represents the probability of the data $(X_{N_0},Y_{N_0})$ given the model (i.e., the depth $L$, layer widths $N_1,\ldots, N_L$ and prior variance $\sigma^2$).  As detailed in \S 4, \S 5 of \cite{mackay1992bayesian}, maximizing the Bayesian model evidence is therefore equivalent to maximum likelihood estimation over the space of models and gives a principled way to select among different models, all of which interpolate the training data. Before stating our technical results we briefly explain how to compute effective depth and how to reason about optimal extrapolation in linear networks.

\subsection{Effective Depth}
The number of layers $L$ does not provide a useful measure of complexity for the prior distribution over network outputs $f(x)=\theta^Tx$ when $\theta$ is drawn from the deep linear prior \eqref{eq:prior-def}. This is true in both linear and non-linear networks at large width (see e.g., \cite{hanin2018neural, hanin2019finite, hanin2022random, roberts2022principles} for a treatment of deep non-linear networks). A more useful notion of depth is 
\begin{equation}\label{eq:lpre-def}
\lpre =\text{ effective depth of prior }:= \sum_{\ell=1}^L \frac{1}{N_\ell},
\end{equation}
and it is indeed $\lpre$ that plays an important role in our results. Let us provide a brief justification for why $\lpre$ is a natural measure of complexity for the prior. With $N_1=\cdots=N_L=N$, Theorem 1.2 in \cite{hanin2021non} shows that when $\sigma^2=1$, under the prior, the squared singular values of $\lr{W^{(L)}\cdots W^{(1)}}^{_{1/L}}$ converge to the uniform distribution on $[0,1]$. Hence, only the squared singular values of $\lr{W^{(L)}\cdots W^{(1)}}^{_{1/L}}$ lying in intervals of the form $[1-CL^{-1},1]$ correspond to singular values of $W^{(L)}\cdots W^{(1)}$ that remain uniformly bounded away from $0$ at large $L$. At least heuristically, this implies that $W^{(L)}\cdots W^{(1)}$ is supported on matrices of rank approximately $\lpre^{-1}$.\footnote{We do not know this for sure since uniformity for the distribution of singular values at the right edge of the spectrum does not follow from a result only about the global density of singular values.} 

Viewing $\lpre^{-1}$ as a natural measure of the number of degrees of freedom in the prior motivates the introduction of a \emph{posterior effective depth}
\begin{equation}\label{eq:lpost-def}
    \lpost:=\frac{P}{\lpre^{-1}}= \sum_{\ell=1}^L \frac{P}{N_\ell}.
\end{equation}
$\lpost$ is a ratio between the number of degrees of freedom in the training data (given by the number of training data points) and in the prior. We'll see in Theorem \ref{thm:PLN} that it is precisely $\lpost$ that controls the structure of the posterior. 

\subsection{Optimal Extrapolation}
This section explains key structural properties of Bayesian posteriors in linear networks. Consider the model $f(x) = \theta^T x$ and decompose the parameters  $\theta = \theta_{||}+\theta_\perp$
as in Figure \ref{fig:perp}. Since zero noise posteriors fit the training data, we have  
\[
\theta\sim \Papost\qquad \Longrightarrow \qquad \theta_{||}=\theta_{*,N_0},
\]
where $\theta_{*,N_0}$ is the minimum-norm interpolant \eqref{eq:min-theta}. Moreover, the prior over $\theta$ is invariant under all orthogonal transformations and the likelihood is invariant under arbitrary transformations of $\theta_\perp$. Hence, in distribution
\begin{equation}\label{eq:a-form}
\theta\sim \Papost\qquad \Longrightarrow \qquad \theta_\perp ~\stackrel{d}{=}~u\cdot \norm{\theta_\perp},
\end{equation}
where $u$ is independent of $\norm{\theta_\perp}$ and is uniformly distributed on the unit sphere in  $\col(X_{N_0})^\perp\subseteq \R^{N_0}$. The only degree of freedom in the posterior is therefore the distribution of the radial part $\norm{\theta_\perp}$ of the vector $\theta_\perp$. Given a test data point $x = x_{||} + x_\perp$, we find
\[
\theta\sim \Papost\qquad \Longrightarrow \qquad f(x) = (\theta_{*,N_0})^T x_{||} + \|\theta_\perp\|\cdot u^T x_\perp.
\]
The distribution of $\|\theta_\perp\|$ controls the scale of predictions for data not spanned by the training set, i.e., for the task of extrapolation. For example if $x=x_{||}\in \col(X_{N_0})$, then $f(x)$ has zero variance since it is determined completely by the training data. More generally, by the Poincare-Borel Theorem (see \cite{diaconis1987dozen,borel1914introduction}) we have for $N_0, P \gg 1$ that
\begin{equation}\label{eq:post-form}
f(x) ~\approx~ \mathcal N\lr{\lr{\theta_{*,N_0}}^Tx_{||},~\norm{\theta_\perp}^2\frac{\norm{x_\perp}^2}{N_0-P}}.
\end{equation}
Indeed, since $\widehat{\theta}_\perp=\theta_\perp/\norm{\theta_\perp}\in \R^{N_0 - P}$ is rotationally invariant under the posterior, the Poincare-Borel Theorem (e.g. Theorem 2 in [21] together with the fact that the mixing measure $\mu$ is precisely a point-mass at $1$ by (1) in \cite{diaconis1987dozen}) shows that the joint distribution of any fixed (or even slowly growing) number of marginals $\set{\widehat{\theta}_\perp^T x_1,\ldots, \widehat{\theta}_\perp^Tx_k}$ is approximately Gaussian when $N_0 - P$ is large. In particular, since predictions under the posterior are of the  $f(x;\theta)=\theta_{||}^Tx_{||} + \norm{\theta_\perp}\widehat{\theta}_\perp^Tx_{\perp}$ and $\norm{\theta_\perp}$ is independent of $\widehat{\theta}_\perp$, we see that in the high-dimensional regime $N_0, P\gg 1$ the finite-dimensional distributions of the posterior are approximately normal.

At $L=0$, the prior and posterior distributions over $\|\theta_\perp\|^2$ are identical, preventing any feature learning from occurring. For $L \geq 1$, all components of $\theta$ are correlated under the prior, allowing information from the training data to be encoded into $\theta_\perp$. We shall see (Theorem \ref{thm:PLN}) that $\lpost$ quantifies how much information the model learns about $\theta_\perp$. In particular, increasing $\lpost$ causes the posterior distribution of $\|\theta_\perp\|^2$ to be more and more concentrated around a particular value:
\[
\|\theta_\perp\|^2~\approx~ \norm{\theta_{*,N_0}}^2(1-\alpha_0)/\alpha_0.
\]
This special choice of scale maximizes Bayesian evidence, in accordance with the first takeaway described in the Introduction. It corresponds to the natural estimate for the true signal strength $\norm{v}^2$ under a zero noise generative process $y_i=v^Tx_i$ in which $x_i$ are isotropic, given that one observes only the projection $v_{||}=\theta_{*,N_0}$ of $v$ onto directions in the training data.

\section{Main Results}
\subsection{Non-Asymptotic Results}\label{sec:Z-form}
We are now ready to formulate our first main result (Theorem~\ref{thm:Z-form}), which expresses the partition function $Z_\infty({\bf t})$ defined in \eqref{eq:part-def} in terms of the Meijer-G function; this allows the Bayesian evidence and predictive posterior to be written in exact closed form for any choice of network depth, hidden layer widths, dataset size, and input dimension. Compared to prior work using either iterative saddle point approximations of integrals encoding the Bayesian posterior~\cite{li2021statistical} or more involved methods such as the replica trick~\cite{zavatone2022contrasting}, our method provides a direct representation of the network posterior via the partition function in terms of a single contour integral without specializing to limiting cases. Additional quantities, such as the variance of the posterior, are simply expressed as a ratio of Meijer-G functions. We shall later recover known limiting cases and uncover new asymptotic results from expansions of the Meijer-G function (Theorem~\ref{thm:logG}).
\begin{theorem}[Predictive Posterior and Evidence]\label{thm:Z-form}
Fix $P,L,N_0,\ldots, N_L\geq 1,\, \sigma^2>0$ as well as training data $X_{N_0},Y_{N_0}$. Fix ${\bf t}\in \R^{N_0}$ and write
\[
{\bf t} = {\bf t}_{||}+ {\bf t}_\perp,\qquad {\bf t}_{||}\in \col(X_{N_0}),\quad {\bf t}_{\perp}\in \col(X_{N_0})^\perp.
\]
Define $4M = \prod_{\ell=0}^L 2\sigma^2/N_\ell$
%\begin{align}\label{eq:M-def}
%4M := \prod_{\ell=0}^L %\frac{2\sigma^2}{N_\ell}.
%\end{align}
and introduce the following shorthand for the Meijer-G functions (see \S \ref{sec:back}) %$\G{m,n}{p,q}{z}{a_1,\ldots, a_p}{b_1,\ldots, b_q}$ 
with parameters given by layer widths:
\begin{align*}
&\G{L+1,0}{0,L+1}{\frac{\norm{\theta_{*,N_0}}^2}{4M}}{-}{\frac{P}{2},\frac{\mathbf{N}}{2}+k}:=\qquad\G{L+1,0}{0,L+1}{\frac{\norm{\theta_{*,N_0}}^2}{4M}}{-}{\frac{P}{2},\frac{N_1}{2}+k,\ldots, \frac{N_L}{2}+k}.
\end{align*}
The partition function $Z_\infty({\bf t})$ of the predictive posterior defined in \eqref{eq:part-def} is
\begin{align}
&Z_\infty({\bf t}) =  \lr{\frac{4\pi }{\norm{\theta_*}^2}}^{\frac{P}{2}} \exp\left[-i\inprod{\theta_{*,N_0}}{{\bf t}}\right]\prod_{\ell=1}^L \Gamma\lr{\frac{N_\ell}{2}}^{-1} \nonumber\\
&\qquad \qquad \times \sum_{k=0}^\infty \frac{(-1)^k}{k!}\big|\big|{\bf t}_{\perp} \big|\big| ^{2k} M^{k} \G{L+1,0}{0,L+1}{\frac{\norm{\theta_{*,N_0}}^2}{4M}}{-}{\frac{P}{2}, \frac{\bf{N}}{2}+k}.
\end{align}
In particular, the Bayesian model evidence $Z_\infty({\bf 0})$ equals
\begin{align}
\label{eq:evidence-form} \frac{(4\pi)^{P/2}}{\norm{\theta_*}^{P}\prod_{\ell=1}^L \Gamma\lr{\frac{N_\ell}{2}}} \G{L+1,0}{0,L+1}{\frac{\norm{\theta_{*,N_0}}^2}{4M}}{-}{\frac{P}{2}, \frac{\bf{N}}{2}}.
%\label{eq:evidence-form}Z_\infty({\bf 0}) &= \lr{\frac{4\pi }{\norm{\theta_*}^2}}^{\frac{P}{2}} \prod_{\ell=1}^L \Gamma\lr{\frac{N_\ell}{2}}^{-1} \nonumber \\
%&\quad \times \G{L+1,0}{0,L+1}{\frac{\norm{\theta_{*,N_0}}^2}{4M}}{-}{\frac{P}{2}, \frac{\bf{N}}{2}}.
\end{align}
Further, given $x\in \R^{N_0}$, the mean of the predictive posterior is
\begin{equation}\label{eq:mean-G}
\mathbb E_{\text{post}}\left[f(x)\right]= (\theta_{*,N_0})^Tx,
\end{equation}
while the posterior variance $\Var_{\mathrm{post}}\left[f(x)\right]$ is 
\begin{equation}\label{eq:var-G}
2M\norm{x_{\perp}}^2\frac{\G{L+1,0}{0,L+1}{\frac{\norm{\theta_{*,N_0}}^2}{4M}}{-}{\frac{P}{2}, \frac{\bf{N}}{2}+1}}{\G{L+1,0}{0,L+1}{\frac{\norm{\theta_{*,N_0}}^2}{4M}}{-}{\frac{P}{2}, \frac{\bf{N}}{2}}},    
\end{equation}
where $x_\perp$ is the projection of $x$ onto the orthogonal complement of the span $\col(X_{N_0})$ of the training data. 
\end{theorem}
See \S \ref{sec:Z-form-pf} for a proof.

\subsection{Asymptotic Results}\label{sec:G-results}
To evaluate the predictive posterior and evidence in Theorem~\ref{thm:Z-form} in the limits where $N_0, P, N_\ell$ (and potentially $L$) tend to infinity, we require novel expansions of the Meijer-G function obtained by the Laplace method. We are interested in regimes where $N_0,P$ grow and will assume a mild compatibility condition on the training data: for all $\alpha_0 \in (0, 1)$, we assume there exists constant $\norm{\theta_*}$ such that
\begin{equation}\label{eq:Vstar-def}
\lim_{\substack{P,N_0\gives \infty \\ P/N_0\gives \alpha_0\in (0,1)}}\norm{\theta_{*,N_0}} ~=~ \norm{\theta_*},    
\end{equation}
where the convergence is in distribution. This assumption is very generic and is (for example) satisfied for a Gaussian data model where inputs are Gaussian
\[
X_{N_0}=\lr{x_{i,N_0},\,i=1,\ldots, P},\quad x_{i,N_0}\sim \mathcal N(0,\Sigma_{N_0})\quad\text{iid},
\]
outputs are linear plus noise
\[
Y = V_{N_0} X_{N_0} + \epsilon_{N_0},\; V_{N_0}\sim \mathrm{Unif}\lr{S^{N_0-1}},\; \epsilon_{N_0}\sim \mathcal N\lr{0,\sigma_\epsilon^2 I_{N_0}},
\]
and the spectral density of the design matrices $\Sigma_{N_0}$ converges weakly as $N_0\gives \infty$ to a fixed probability measure on $\R_+$ with finite moments.

To minimize notation, we report here the expansions in terms of a single layer width $N = N_1 = \dots = N_L$, but expansions with distinct $N_\ell$ (and to higher order) are provided in the proof (\S \ref{pf:logG}).

\begin{theorem}[Asymptotic Expansions of Meijer-$G$]\label{thm:logG}
Set $N_1,\ldots,N_L = N$ and define $\mathbf{N} = (N_1, \dots, N_L)$. Suppose that the training data satisfies \eqref{eq:Vstar-def}. In different limiting cases such that $\{P, N\} \to \infty$ with fixed $P/N_0 = \alpha_0$, we evaluate the quantities
\begin{align*}
    \log G &:= \log \G{L+1,0}{0,L+1}{\frac{\norm{\theta_{*,N_0}}^2}{4M}}{-}{\frac{P}{2},\frac{\mathbf{N}}{2}+k}\\
    \Delta(\log G)[k] &:= \log \G{L+1,0}{0,L+1}{\frac{\norm{\theta_{*,N_0}}^2}{4M}}{-}{\frac{P}{2},\frac{\mathbf{N}}{2}+k} - \log \G{L+1,0}{0,L+1}{\frac{\norm{\theta_{*,N_0}}^2}{4M}}{-}{\frac{P}{2},\frac{\mathbf{N}}{2}}.
\end{align*}
We will see in each case that to leading order $\log G$ does not depend on $k$ while $ \Delta(\log G)[k] $ does.
\begin{enumerate}[label=(\alph*)]
\item Fix $L < \infty,\, \alpha,\sigma^2>0$. Suppose $P,N\gives \infty$ with $P/N\gives \alpha$. Then,
%\begin{align}
%    \label{eq:logG-Lfinite}\log G &= \frac{N\alpha}{2}\left[\log\lr{\frac{N\alpha}{2}} + \log\lr{1+\frac{z_*}{\alpha}} - \lr{1+\frac{z_*}{\alpha}}\right] \nonumber \\
%    &\quad + \frac{NL}{2}\left[\log\lr{\frac{N}{2}} + \log\lr{1+z_*} - (1+z_*)\right] \nonumber\\
%    &\quad + \tilde O(1),
%\end{align}
\begin{align}
    \notag \log G &= \frac{N\alpha}{2}\left[\log\lr{\frac{N\alpha}{2}} + \log\lr{1+\frac{z_*}{\alpha}} - \lr{1+\frac{z_*}{\alpha}}\right] \nonumber \\
    &\label{eq:logG-Lfinite}\quad + \frac{NL}{2}\left[\log\lr{\frac{N}{2}} + \log\lr{1+z_*} - (1+z_*)\right]
\end{align}
plus an error of size $\tilde O(1)$ and
\begin{align}
    \label{eq:DlogG-Lfinite} \Delta(\log G)[k] &= kL\left[\log\lr{\frac{N}{2}} + \log(1+z_*)\right],
\end{align}
%\begin{align}
%    \label{eq:DlogG-Lfinite} \Delta(\log G)[k] &= kL\left[\log\lr{\frac{N}{2}} + \log(1+z_*)\right] + \tilde O\left(\frac{1}{N}\right),
%\end{align}
plus an error of size $\tilde O\left(1/N\right)$, where $z_*> \min\set{-\alpha,-1}$ is the unique solution to
\begin{equation}\label{eq:Psi-crit}
\lr{1+\frac{z_*}{\alpha}}\lr{1+z_*}^L = \frac{\norm{\theta_*}^2}{\sigma^{2(L+1)}\alpha_0}.
\end{equation} 
\item Fix $\lpre,\alpha>0$ and suppose $L=\lpre N, \, P,N\gives \infty,\, P/N\gives \alpha, \, \sigma^2=1$. Then,
\begin{align}
    \label{eq:logG-LN} \log G &= \frac{\lpre N^2}{2}\left[\log\lr{\frac{N}{2}} - 1\right] + \frac{N\alpha}{2}\left[\log\lr{\frac{N\alpha}{2}} - 1\right]\\
    &\quad + \frac{\lpre N}{2}\left[-\log\lr{\frac{N}{2}} + \log(2\pi)\right]  + \tilde O(1)\nonumber
\end{align}
and, up to an error of size $\tilde O\left(1/N\right)$, 
\begin{align}
    \label{eq:DlogG-LN} \Delta(\log G)[k] &= k\left[\lpre N \log\lr{\frac{N}{2}} + \log \lr{\frac{\norm{\theta_*}^2}{\alpha_0}}\right].
\end{align}
%\begin{align}
%    \label{eq:DlogG-LN} \Delta(\log G)[k] &= k\left[\lpre N \log\lr{\frac{N}{2}} + \log \lr{\frac{\norm{\theta_*}^2}{\alpha_0}}\right] \nonumber \\
%    &\quad + \tilde O\left(\frac{1}{N}\right).
%\end{align}
\item Fix $\lpost >0$. Suppose $N,P,L\gives \infty$ with $LP/N\gives \lpost, \, \sigma^2=1$ and $L/N\to 0$. Then,
\begin{align}
    \label{eq:logG-PLN} \log G &= \frac{P}{2}\left[\log\lr{\frac{P}{2}}-1\right] + \frac{\lpost N}{2}\frac{N}{P}\left[\log\lr{\frac{N}{2}}-1\right] \nonumber \\
    &\quad + \frac{P}{2}\left[\log(1+t_*)-t_*\lr{1+\frac{\lpost t_*}{2}}\right] \nonumber \\
    &\quad + \frac{\lpost N}{2}\left[\frac{N}{P}\lr{1+\frac{P}{N}t_*}\log\lr{1+\frac{P}{N}t_*} - t_*\right] \nonumber\\
    &\quad + \tilde O(1),
\end{align}
and
\begin{align}
    \label{eq:DlogG-PLN}   \Delta(\log G)[k] &= k\left[L \log \lr{\frac{N}{2}} + \lpost t_*\right] + \tilde{O}\left(\frac{1}{N}\right),
\end{align}
where $t_*$ is the unique solution to
\begin{align}
    e^{\lpost t_*}(1+t_*) &= \norm{\theta_*}^2/\alpha_0.
\end{align}
\end{enumerate}
In all the estimates above, $\tilde O(1) = O(\max\{\log P, \log N\})$ suppresses lower-order terms of order 1, up to logarithmic factors; similarly, $\tilde O(1/N) = O(\max\{\log P, \log N\}/N)$. Suppressed terms are included in \S \ref{pf:logG}
\end{theorem}

\subsection{Model Selection and Extrapolation}\label{sec:limit-thms}
We combine our non-asymptotic formulas for the posterior and evidence from Theorem \ref{thm:Z-form} with the Meijer-G function expansions from Theorem \ref{thm:logG} to investigate two fundamental questions about Bayesian interpolation with linear networks. Together they give, at least in the restricted setting of deep linear networks with output dimension $1$, a principled reason to prefer deeper networks.

\begin{itemize}
\item \textbf{Model Selection.} How should we choose the prior weight variance $\sigma^2$, model depth $L$ and layer widths $N_\ell$? Recall that the Bayesian model evidence $Z_\infty({\bf 0})$ from \eqref{eq:char-part} represents the likelihood of observing the data  $(X_{N_0},Y_{N_0})$  given the architecture $L,N_\ell$ and the prior weight variance $\sigma^2$. Maximizing the Bayesian evidence is therefore maximum likelihood estimation over the space of models and gives a principled method for model selection. We shall see that data-agnostic priors maximize the evidence for networks with infinite $\lpost$, while shallower networks require data-dependent priors.
\item \textbf{Extrapolation.} How optimal is the posterior of a linear network? Predictions on inputs from directions orthogonal to the training data are determined by the distribution $\theta_\perp$. At $\lpost = 0$, its prior and posterior coincide. As $\lpost$ increases, however, we shall find that information from the training set mixes into $\theta_\perp$ and ultimately maximizes evidence, producing optimal extrapolation.
\end{itemize}

For simplicity, we report our results here in terms of a single hidden layer width $N=N_1=\cdots=N_L$. The general form with generic $N_\ell$, as well as related results not stated here, are available by direction application of the general expansions in \S \ref{sec:bayesfeature-pf}. For convenience, we define
\begin{align}
    \nu &:= \norm{\theta_*}^2/\alpha_0.
\end{align}
As elsewhere, we emphasize that we focus on the regime
\[
P/N_0\gives \alpha_0\in (0,1)
\]
as $P,N_0\to\infty$. When $\alpha_0\geq 1$, Theorem \ref{thm:Z-form} still holds but immediately yields that $\theta=\theta_{*,N_0}$ almost surely since $\theta_\perp =0$.
As described by \eqref{eq:post-form} and stated rigorously in \S \ref{sec:setup}, the predictive posterior in the regime $P < N_0$ is Gaussian with variance determined by $\norm{\theta_\perp}^2$, which converges to a constant. We rewrite this posterior in the form
\[
f(x) \quad \to\quad  \mathcal{N}(\mu_*, \nu c\Sigma_\perp),\qquad x = \lr{x_{i,N_0},\,i=1,\ldots, k}
\]
for free scalar $c$, scalar $\mu_*$, and $k\times k$ PSD matrix $\Sigma_\perp$ given by
\[
\mu_* := \inprod{\theta_{*,N_0}}{x}, \quad \Sigma_\perp = \frac{\inprod{x_{i,N_0}^\perp}{x_{j,N_0}^\perp}}{N_0-P}
\]
in the limit $P,N_0\to\infty$ and $P/N_0\to\alpha_0\in(0,1)$. The following results report the Bayesian evidence and value of $c$ under different join scalings of $P, N$ and $L$. First, we observe that maximizing Bayesian evidence always results in the posterior corresponding to $c=1$ (proven \S \ref{sec:bayesfeature-pf}).

\begin{theorem}[Universal Maximal Evidence Posterior]
\label{thm:bayesfeature}
Fix $L, N_1\dots, N_L \geq 1$, $\sigma^2 > 0$, $\alpha_0\in (0,1)$, and consider sequences of training data sets $X_{N_0}\in \R^{N_0\times P}, Y_{N_0}\in \R^{1\times P}$ such that
\begin{align*}
    P, N_0 \to \infty, \qquad P/N_0 \to \alpha_0
\end{align*}
and \eqref{eq:Vstar-def} holds. Let $Z_\infty({\bf 0})$ denote the limiting Bayesian model evidence. The following statements are equivalent:
\begin{enumerate}[label=(\alph*)]
    \item $\sigma^2$ maximizes $Z_\infty({\bf 0})$, and
    \item the posterior over predictions $f(x)$ converges weakly to a Gaussian $\mathcal{N}\lr{\mu_*, \nu \Sigma_\perp}$.
\end{enumerate}
\end{theorem}

We emphasize that Theorem \ref{thm:bayesfeature} holds for any choice of depth and hidden layer widths. At finite depth, we shall see that identifying the evidence-maximizing prior $\sigma_*^2$ requires knowledge of the data-dependent parameter $\nu$. In contrast, infinite-depth networks have evidence maximized by $\sigma_*^2 = 1$, allowing them to successfully infer $\nu$ from training data despite a data-agnostic prior. We proceed to consider different joint scalings of $P, N$ and $L$ in Theorems~\ref{thm:finite-L}, \ref{thm:LN} and \ref{thm:PLN}, which follow directly from writing the partition function in terms of the Meijer-G function (Theorem~\ref{thm:Z-form}) and applying the suitable asymptotic expansion (Theorem~\ref{thm:logG}).

\subsubsection{Finite Depth}\label{sec:Lfinite-thms}
We present here a characterization of the infinite-width posterior and model evidence for networks with a fixed number of hidden layers.
\begin{theorem}[Posterior and Evidence at Finite $L$]\label{thm:finite-L}
For each $N_0,P\geq 1$, consider training data $X_{N_0},Y_{N_0}$ satisfying \eqref{eq:Vstar-def}. Fix constants $L\geq 0,\,\sigma^2 > 0$ and suppose that
\begin{equation*}
P,N_0,N \gives \infty,\quad P/N_0 \gives \alpha_0\in (0,1),\quad P/N\gives \alpha\in (0,\infty). 
\end{equation*}
In this limit, the posterior over predictions $f(x)$ converges weakly to a Gaussian
\begin{align*}
\mathcal{N}\lr{\mu_*,~\nu\Sigma_\perp\lr{1+\frac{z_*}{\alpha}}^{-1}},
\end{align*}
where $z_*$ is the unique solution to
\begin{equation}\label{eq:zstar}
\frac{\nu}{\sigma^{2(L+1)}} = \lr{1+\frac{z_*}{\alpha}}\lr{1+z_*}^L,\quad z_* > \max\set{-1,-\alpha}.
\end{equation}
Additionally, we have the following large-$P$ expansion of the Bayesian model evidence
\begin{align*}
\log Z_\infty({\bf 0}) &~\gives~ \frac{P}{2}\Bigg[\log\lr{\frac{P}{2}} + \log\lr{1+\frac{z_*}{\alpha}} - \lr{1+\frac{z_*}{\alpha}} \nonumber \\
&\quad - \log\lr{\frac{\norm{\theta_*}^2}{4\pi}}\Bigg] + \frac{NL}{2}\left[\log\lr{1+z_*} - z_*\right] \nonumber\\
&\quad + O(\max\set{\log P, \log N}).
\end{align*}
\end{theorem}

In regimes where $z_*$ tends to zero, the posterior will converge to the evidence-maximizing posterior independent of the architecture and prior. By \eqref{eq:zstar}, this occurs for instance in the limit of infinite depth with $\sigma=1$ or when $\nu=\sigma^{2(L+1)}$. This last condition corresponds to maximizing the Bayesian evidence $Z_\infty({\bf 0})$ at finite $L$.

\begin{corollary}[Bayesian Model Selection at Finite $L$]\label{cor:Lfix-evid}
In the setting of Theorem \ref{thm:finite-L}  the Bayesian evidence $Z_\infty({\bf 0})$ satisfies: 
\begin{align}
\label{eq:sig-star}\sigma_*^2 &= \argmax_{\sigma}Z_\infty({\bf 0}) = \nu^{\frac{1}{L+1}}\\
\label{eq:L-star} L_* &= \argmax_{L}Z_\infty({\bf 0}) =\frac{\log(\nu)}{\log(\sigma^2)}-1.
\end{align}
In particular, given a prior variance with $\sgn{\nu-1}=\sgn{\sigma^2-1}$ satisfying $|\sigma^2-1| \leq \epsilon$, the optimal depth network satisfies $L_* \geq |\log(\nu)|/\epsilon$.
\end{corollary}

To put this Corollary into context, note that in the large-width limit of Theorem \ref{thm:finite-L} there are only two remaining model parameters, $L$ and $\sigma^2$. Model selection can therefore be done in two ways. The first is an empirical Bayes approach in which one uses the training data to determine a data-dependent prior $\sigma_*^2$ given by \eqref{eq:sig-star}. The other approach, which more closely follows the use of neural networks in practice, is to seek a universal, data-agnostic value of $\sigma^2$ and optimize instead the network architecture. The expression \eqref{eq:sig-star} shows that the only way to choose $\sigma^2,L$ to (approximately) maximize model evidence for any fixed $\nu$ is to take $\sigma^2\approx 1$ with $\sgn{\sigma^2-1}=\sgn{\nu-1}$ and $L\gives \infty$. Hence, restricting to the data-agnostic prior $\sigma^2=1$ naturally leads to a Bayesian preference for infinite-depth networks, regardless of the training data. This motivates us to consider large-$N$ limits in which $L$ tends to infinity, which we take up in the next two sections.

\subsubsection{Infinite Depth}\label{sec:Linfinit-thms}
We fix $\sigma^2=1$ and investigate extrapolation and model selection in regimes where $N,L,P\gives \infty$ simultaneously.

\begin{theorem}[Posterior and Evidence at Fixed $\lpre$]
\label{thm:LN}
For each $N_0,P\geq 1$, consider training data $X_{N_0},Y_{N_0}$ satisfying \eqref{eq:Vstar-def}. Moreover, fix $\lpre,\alpha \in (0,\infty),\, \alpha_0\in (0,1).$ Suppose that $N_1=\cdots=N_L=N$ and that
\begin{equation*}
P,N_\ell,L\gives \infty,\quad P/N_0\gives \alpha_0,\quad P/N\gives \alpha,\quad L/N\gives \lpre.
\end{equation*}
In this limit, the posterior over predictions $f(x)$ converges weakly to a Gaussian
\begin{align*}
\mathcal{N}\lr{\mu_*,~\nu\Sigma_\perp},
\end{align*}
which is independent of $\alpha,\lpre$. Additionally, the evidence admits the following asymptotic expansion:
\begin{align*}
\log Z_\infty({\bf 0}) &= \frac{P}{2}\left[\log\lr{\frac{P}{2}} - 1 - \log\lr{\frac{\norm{\theta_*}^2}{4\pi}}\right] + O(\max\set{\log P, \log N}).
\end{align*}
\end{theorem}This result highlights the remarkable nature of data-driven extrapolation in deep networks.
\begin{corollary}[Optimal Learning by Deep Networks]\label{cor:feature-learning}
In the setting of Theorem \ref{thm:LN}, the posterior distribution over regression weights $\theta$ is the same in the following two settings:
\begin{itemize}
	\item We fix $L\geq 0$, take the data-dependent prior variance $\sigma_*$ that maximizes the Bayesian model evidence as a function of the training data and network depth as in \eqref{eq:sig-star}, and send the network width $N$ to infinity.
	\item We fix $\lpre >0$, take a data-agnostic prior $\sigma^2=1$, and send both the network depth $L:=\lpre \cdot N$ and network width $N$ to infinity together.
\end{itemize}
\end{corollary}
This corollary makes precise the statement that infinitely deep networks with data-agnostic priors performs as optimally finite depth networks with fine-tuned data-dependent priors.

In the \S \ref{sec:evidence}, we find that the ratio of model evidence at fixed depth to the model evidence at infinite depth vanishes like $\exp[-O(N)]$. In comparison, mis-specifying the value of constant $\lpre$ only results in an $O(1)$ ratio of model evidences, and it does not affect the posterior to leading order. We conclude that, at $\sigma^2=1$, wide networks with depth comparable to width are robustly preferred to shallow networks.

\subsubsection{Scaling Laws for Optimal Learning}\label{sec:scaling-thms}
To emphasize the similarity between dataset size and depth, we take the limit of $L,P,N\to\infty$ while holding $LP/N$ constant. This results in a scaling law for feature learning that only depends on $\lpost$, as seen in the following theorem.

\begin{theorem}[Posterior and Evidence at Fixed $\lpost$]
\label{thm:PLN}
For each $N_0,P\geq 1$, consider training data $X_{N_0},Y_{N_0}$ satisfying \eqref{eq:Vstar-def}. Moreover, fix constants $\lpost > 0,\, \alpha_0\in (0,1), \, \sigma^2=1$. Suppose that $N_1=\cdots = N_L=N$ and suppose that
\begin{align*}
P, N_\ell, L \to \infty, \; \frac{P}{N_0} \to \alpha_0 , \; \frac{P}{N} \to 0, \; \frac{L}{N} \to 0, \; \frac{LP}{N} \to \lpost.
\end{align*}
In this limit, the posterior over predictions $f(x)$ converges weakly to a Gaussian $\mathcal{N}\lr{\mu_*,~\nu\Sigma_\perp(1+t_*)^{-1}}$,
%\begin{align*}
%\mathcal{N}\lr{\mu_*,~\nu\Sigma_\perp(1+t_*)^{-1}},
%\end{align*}
where $t_*$ is the unique solution to
\begin{align*}
    \nu = (1+t_*)e^{\lpost t_*}, \quad t_* > -1.
\end{align*}
Additionally, we have the following asymptotic expansion for the Bayesian model evidence
\begin{align*}
    \log Z_\infty({\bf 0}) &= \frac{P}{2}\lr{\log\lr{\frac{P}{2}}-1} - \frac{P}{2} \log\lr{\frac{\norm{\theta_*}^2}{4\pi}}\\
    &\quad + \frac{P}{2}\left[\log(1+t_*)-t_* - \frac{1}{2}\lpost t_*^2\right] + \tilde O(\max\set{\log P, \log N, \log L}).
\end{align*}
\end{theorem}
In the limit $\lpost\to\infty$, Theorem~\ref{thm:PLN} implies that optimal feature learning is rapidly approached --- specifically, like the harmonic mean of 1 and $\lpost/\log \nu$. Bayesian model selection also drives $\lpost$ to infinity.
\begin{corollary}[Bayesian Model Selection at Fixed $\lpost$]
In the setting of Theorem~\ref{thm:PLN}, the Bayesian evidence $Z_\infty({\bf 0})$ is monotonically increasing in $\lpost$. Specifically,
\begin{align}
\frac{\partial \log Z_\infty({\bf 0})}{\partial \lpost} = \frac{Pt_*^2}{4} \geq 0.
\end{align}
\end{corollary}
These results provide simple scaling laws that apply independently of the choice of dataset, demonstrating the behavior of the predictor's posterior distribution and model evidence in terms of $\lpost$. The coupling of depth and dataset size in $\lpost$ provides a novel interpretation of depth as a mechanism to improve learning in a manner similar to additional data: larger datasets and larger depths contribute equally towards aligning the prior $\sigma^2=1$ towards the correct posterior.

\subsection{Properties of Deep Linear Networks}\label{sec:properties}
We relate our work to prior results in the literature: variance-limited scaling laws~\cite{bahri2021explaining} and sample-wise double descent~\cite{belkin2019reconciling, belkin2020two,zavatone2022contrasting}. To make the discussion concrete, we shall focus on the architecture introduced in Theorem~\ref{thm:LN}, where depth, width, and dataset size scale linearly with each other to produce a posterior that exhibits optimal feature learning given prior $\sigma^2=1$.

\subsubsection{Variance-Limited Scaling Laws}
We examine the scaling behavior of model error in the infinite-width or infinite-dataset limits. The work of~\cite{bahri2021explaining} shows that the difference between the finite-size loss and the infinite-size loss scales like $1/x$ for $x = N$ or $P$ while the other parameter is held fixed ($P$ or $N$, respectively). Here, we demonstrate an analogous scaling law when $N \propto P \propto L$. Similarly to the results of~\cite{bahri2021explaining}, this scaling law is independent of the choice of dataset and consequently provides a universal insight into how performance improves with larger models or more data. The proof is found in \S \ref{sec:scaling-pf}.

\begin{theorem}[Variance-Limited Scaling Law]\label{thm:scaling}
For each $N_0\geq 1$ consider training data $X_{N_0},Y_{N_0}$ satisfying \eqref{eq:Vstar-def}. Fix constants $\lpre > 0,\alpha>0, \, \sigma^2=1$. Suppose that $N_1=\cdots=N_L= N$, that $L=\lpre N$ and number of data points satisfies $P=N\alpha$. Then, as $N\gives \infty$,
\begin{align*}
    \Var_{\mathrm{post}}\left[f(x)\right] = \lim_{N\to\infty} \Var_{\mathrm{post}}\left[f(x)\right] +\frac{C}{N}+O\lr{\frac{\log N}{N^2}},%&\propto \frac{1}{N} \propto \frac{1}{P} \propto \frac{1}{L}
\end{align*}
where $C\in \R$ is a universal constant.
\end{theorem}

\subsubsection{Double Descent}
We demonstrate double descent in $\alpha_0=P/N_0$ consistent with previous literature~\cite{zavatone2022contrasting}. As a concrete example, we shall consider a Gaussian data model and evaluate double descent for the posterior of optimal feature learning; this posterior is achieved by, for example, deep networks with $\sigma^2=1$ (Theorem \ref{thm:LN}), or finite-depth networks with data-tuned priors $\sigma_*^2$ (Theorem \ref{thm:finite-L}). The proof is found \S \ref{sec:dd-pf}.

\begin{theorem}[Double Descent in $\alpha_0$]\label{thm:dd}
Consider generative data model
\[
x_i\in \R^{N_0}\sim \mathcal N(0,\mathrm{I}),\; y_i = V_0x_i+\epsilon_i,\; \epsilon_i\sim \mathcal N(0,\sigma_\epsilon^2),\; \norm{V_0}^2=1
\]
and posterior distribution $\mathcal{N}\lr{\mu_*, \nu\Sigma_\perp}$. We have error
\[
\mathbb E_{x,X,\epsilon}\left[\left\langle f(x)-V_0x \right \rangle^2\right] = \begin{cases}
\frac{1}{\alpha_0} - \alpha_0 + \frac{1}{1-\alpha_0}\sigma_\epsilon^2,&\quad \alpha_0<1\\
\frac{1}{\alpha_0-1}\sigma_\epsilon^2,&\quad \alpha_0 \geq 1\end{cases},
\]
which diverges at $\alpha_0 = 1$.
\end{theorem}

\section{Discussion}
Neural networks are non-linear functions of their parameters, making an analytic understanding of their properties difficult. Here, we adopted the simplification of studying linear neural networks, which remain non-linear in their parameters but are more tractable. To conclude, we emphasize three limitations of our work. First, we consider only \textit{linear networks}, which are linear as a function of their inputs. However, they are not linear as a function of their parameters, making learning by Bayesian inference non-trivial. Second, we study learning by Bayesian inference and leave extensions to learning by gradient descent to future work. Finally, our results characterize the \textit{predictive posterior}, i.e., the distribution over model predictions after inference. It would be interesting to derive the form of the posterior over network weights as well.

\subsection{Acknowledgements} This work started at the 2022 Summer School on the Statistical Physics of Machine Learning held at \'Ecole de Physique des Houches. We are grateful for the wonderful atmosphere at the school and would like to express our appreciation to the session organizers Florent Krzakala and Lenka Zdeborov\'a as well as to Haim Sompolinsky for his series of lectures on Bayesian analysis of deep linear networks. We further thank Edward George for pointing out the connection between our work and the deep Gaussian process literature. Finally, we thank Matias Cattaneo, Isaac Chuang, David Dunson, Jianqing Fan, Aram Harrow, Jason Klusowski, Cengiz Pehlevan, Veronika Rockova, and Jacob Zavatone-Veth for their feedback and suggestions. BH is supported by NSF grants DMS-2143754, DMS-1855684, and DMS-2133806. AZ is supported by the Hertz Foundation, and by the DoD NDSEG. We are also thank two anonymous reviewers for improving aspects of the exposition and for pointing out a range of typos in the original manuscript.

\subsection{Background}\label{sec:back}
In this section, we introduce background and results needed for our proofs. Specifically, \S \ref{sec:gammafn} recalls basic definitions and asymptotic expansions for Gamma and Digamma functions (\S \ref{sec:gammafn}), \S \ref{sec:gammarv} recalls the moments of Gamma random variables, and \S \ref{sec:G-def} defines Meijer-G functions and introduces their basic properties.

\subsubsection{Gamma and Digamma Functions}\label{sec:gammafn} We will need the following asymptotic expansions for Euler's Gamma function 
\[
\Gamma(z) = \int_0^\infty t^{z-1}e^{-t}dt,\qquad \R(z) >0
\]
and Digamma function $\phi^{(1)}(z)=\frac{d}{dz}\log \Gamma(z)$ (see Equations 6.1.7, 6.3.13 in \cite{abramowitz1964handbook}):
\begin{proposition}\label{prop:gamma-exp}
We have the following analytic asymptotic expansion for the Gamma function
\begin{equation}\label{eq:Gamma-exp}
\log \Gamma\lr{z} \sim \lr{z-\frac{1}{2}}\log(z) - z + \text{const}+O\lr{\abs{z}^{-1}},\quad \text{as }\abs{z}\gives \infty.
\end{equation}
In particular, we also have
\begin{equation}\label{eq:diGamma-exp}
\phi^{(1)}(z)=\frac{d}{dz}\log \Gamma\lr{z} \sim \log(z) + O(\abs{z}^{-1}),\quad \text{as }\abs{z}\gives \infty.
\end{equation}
Both expansions hold uniformly on sets of the form $\abs{z}< \pi - \delta$ for a fixed $\delta >0$.
\end{proposition}

\subsubsection{Gamma Random Variables}\label{sec:gammarv}
We will need the following well-known exact formulas for the moments of Gamma random variables, which follow directly for the formula for their density:
\[
\mathrm{Den}_{X}(x)=\begin{cases}\frac{x^{k-1}}{\Gamma(k)\theta^{k}}e^{-x/\theta},&\quad x>0\\
0,&\quad x \leq 0\end{cases},\qquad X\sim \Gamma(k,\theta), \quad k,\theta>0
\]
and the definition of the Gamma function.
\begin{proposition}
Let $k,\theta >0$ and suppose $\phi \sim \Gamma(k,\theta)$. Then for any $t\in \R$ we have
\begin{equation}\label{E:gamma-moments}
\E{\phi^t} = \theta^t\frac{\Gamma\lr{k+t}}{\Gamma\lr{k}}.    
\end{equation}
In particular
\begin{equation}\label{E:gamma-MGF}
\E{e^{-it \log\phi}} = \exp\left[-it\log \theta + \log \Gamma\lr{k-it}-\log \Gamma(k)\right]     
\end{equation}
is a meromorphic function of $t$ with poles on the negative imaginary axis:
\[
t = -i(\nu +k),\quad \nu = 0,1,2,\ldots.
\]
\end{proposition}

\subsubsection{Meijer G-Functions}\label{sec:G-def}
The Meijer-G function is defined as the contour integral
\begin{equation}\label{eq:G-def}
\G{p,q}{m,n}{z}{{\bf a}}{{\bf b}} = \frac{1}{2\pi i}\int_{\mathcal C} z^s \chi(s)ds,    
\end{equation}
where
\[
\chi(s):=\frac{\prod_{j=1}^m \Gamma\lr{b_j-s}\prod_{k=1}^n\Gamma\lr{1-a_k+s} }{\prod_{j=m+1}^q \Gamma\lr{1-b_j+s}\prod_{k=n+1}^p \Gamma\lr{a_k+s}}
\]
and $\mathcal C$ is a Mellin-Barnes contour in the complex plane that separates the poles of $\Gamma(b_j-s)$ from those of $\Gamma(1-a_k+s)$ (see Figure \ref{fig:G-contour}). 
\begin{figure}
    \centering
    \includegraphics[scale=.8]{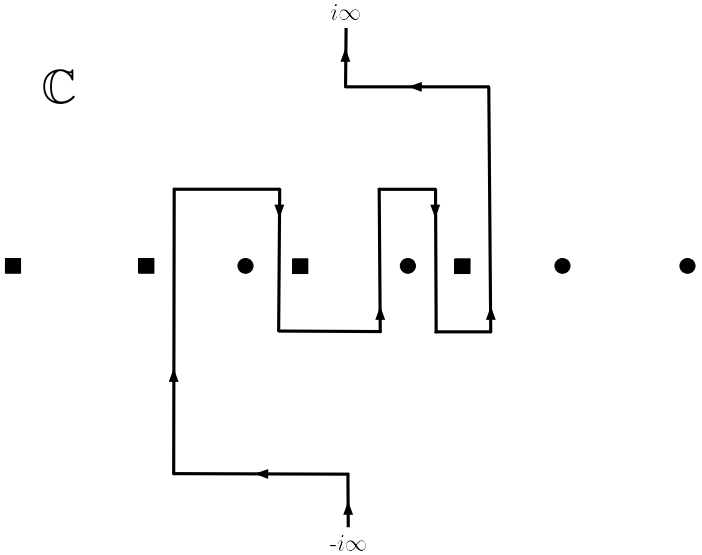}
    \caption{This article concerns only the case when $a_k,b_j$ are real. The poles of $\Gamma(b_j-s)$ (black dots) extend to infinity along the positive real axis when $b_j$ are real. Similarly, the poles of $1-a_k+s$ (black squares) extend to infinity along the negative real axis when $b_j$ are real. The contour of integration (solid black line)  in \eqref{eq:G-def} starts at $-\infty$, keeps the poles of $\Gamma(b_j-s)$ to the right, keeps the poles of $\Gamma(1-a_k+s)$ to the left, and ends at $+i\infty$.}
    \label{fig:G-contour}
\end{figure}

We will need several properties of Meijer-G functions, which we now recall.
\begin{proposition}
For $\eta,\omega>0$, We have
\begin{align}
\label{E:G-cancel}  & \G{m,n}{p,q}{x}{\alpha,{\bf a}}{{\bf b},\alpha} =\G{m,n-1}{p-1,q-1}{x}{{\bf a}}{{\bf b}}  \\
  \label{E:G-weighted-Laplace}  &\int_0^\infty e^{-\w x} x^{-\alpha} \G{m,n}{p,q}{\eta x}{{\bf a}}{{\bf b}} dx=\w^{\alpha-1} \G{m,n+1}{p+1,q}{\frac{\eta}{\w}}{\alpha, {\bf a}}{{\bf b}} \\
  \label{E:G-inversion}&\G{m,n}{p,q}{z}{{\bf a}}{{\bf b}} =\G{n,m}{q,p}{z^{-1}}{1-{\bf b}}{1-{\bf a}}\\
  \label{E:G-polyprod}&z^\rho \G{m,n}{p,q}{z}{{\bf a}}{{\bf b}} =\G{m,n}{p,q}{z}{\rho+{\bf a}}{\rho+{\bf b}}\\
  \label{E:G-BesselJ} &\G{1,0}{0,2}{\frac{x^2}{4}}{-}{\frac{\nu}{2}, -\frac{\nu}{2}}= J_{\nu}(x)\\
  \label{E:G-deriv}
  &z^h \frac{d^h}{dz^h}\G{m,n}{p,q}{z}{{\bf a}}{{\bf b}} = \G{m,n+1}{p+1,q+1}{z}{{0, \bf a}}{{\mathbf{b}, h}}\\
  \label{E:G-conv}
  &\int_0^\infty \G{m,n}{p,q}{\eta x}{\mathbf a}{\mathbf b} \G{\mu, \nu}{\sigma, \tau}{\omega x}{\mathbf c}{\mathbf d} dx = \frac{1}{\eta}\G{n+\mu,m+\nu}{q+\sigma,p+\tau}{\frac{\omega}{\eta}}{-b_1,\dots,-b_m, \mathbf c, -b_{m+1},\dots,-b_q}{-a_1,\dots,-a_n, \mathbf d, -a_{n+1},\dots,-a_p}\\
 \label{E:G-series} &\int_0^\infty \tau^{\alpha-1}\G{s,t}{u,v}{\sigma+\tau}{c_1,\ldots,c_u}{d_1,\ldots, d_v} \G{m,n}{p,q}{\omega \tau}{a_1,\ldots, a_p}{b_1,\ldots, b_q}d\tau  \\
 \notag &= \sum_{k=0}^\infty \frac{\lr{-\sigma}^k}{k!} \G{m+t,n+s+1}{p+v+1,q+u+1}{\omega }{1-\alpha,a_1,\ldots, a_n, k-\alpha-d_1+1,\ldots, k-\alpha-d_v+1,a_{n+1},\ldots, a_p}{b_1,\ldots, b_m, k-\alpha-c_1+1,\ldots, k-\alpha-c_u+1, k-\alpha+1, b_{m+1},\ldots, b_q}
%\label{E:G-series}  \G{m,n}{q,p}{z}{{\bf a}}{{\bf b}}&=\sum_{h=1}^m \frac{\prod_{j=1}^{*m} \Gamma\lr{b_j-b_h} \prod_{j=1}^n \Gamma\lr{1+b_h-a_j}}{\prod_{j=m+1}^q \Gamma\lr{1+b_h-b_j}\prod_{j=n+1}^{*p} \Gamma\lr{a_j-b_h}}z^{b_h}.
\end{align}
For \eqref{E:G-series} we require
\[
\abs{\arg(\sigma)}, \abs{\arg(\omega)}<\pi
\]
and
\[
-\min\set{\Re(b_1),\ldots, \Re(b_m)}
< \Re(\alpha)<2 - \max\set{\Re(a_1),\ldots, \Re(a_n)}- \max\set{\Re(c_1),\ldots, \Re(c_t)}.
\]
Finally, fix $L\geq 1$ and 
\[
k_\ell, \theta_\ell >0,\qquad \ell=1,\ldots, L
\]
and let 
\[
\phi_\ell \sim \Gamma\lr{k_\ell,\theta_\ell}  \quad \text{independent}.
\]
We have for every $y>0$
\begin{align}
    \label{eq:G-gamma} \prod_{\ell=1}^L \Gamma\lr{k_\ell}^{-1 } \frac{1}{A}\G{L,0}{0,L}{\frac{y}{A}}{-}{k_1-1,\ldots, k_L-1} = \mathrm{Den}_{\prod_{\ell=1}^L\phi_\ell}(y),\qquad A:=\prod_{\ell=1}^L \theta_\ell.
\end{align}
\end{proposition}

\subsection{Setup}\label{sec:setup}
As described in the main text, rather than working with the posterior distribution
\begin{align}
\label{eq:post-def-si}d\Papost\lr{\theta~|~L,N_\ell, \sigma^2, X_{N_0}, Y_{N_0}}:= \lim_{\beta\gives \infty}\frac{d\Paprior\lr{\theta~|~N_0,L,N_\ell, \sigma^2}\exp\left[-\beta\mathcal L(\theta~|~X_{N_0},Y_{N_0})\right]}{Z_\beta\lr{X_{N_0},Y_{N_0}~|~L,N_\ell, \sigma^2}},
\end{align}
we work with the characteristic function of the posterior:
\begin{equation}\label{eq:char-part-si}
\mathbb E_{\mathrm{post}}\left[\exp\left\{-i{\bf t}\cdot \theta\right\}\right]=\frac{Z_\infty({\bf t}~|~L,N_\ell, \sigma^2,X_{N_0},Y_{N_0})}{Z_\infty({\bf 0}~|~L,N_\ell, \sigma^2,X_{N_0},Y_{N_0})} ,\qquad {\bf t} = \lr{t_1,\ldots, t_{N_0}}\in \R^{N_0}.
\end{equation}
Here, $\mathbb E_{\mathrm{post}}[\cdot]$ is the expectation with respect to the posterior \eqref{eq:post-def-si}. We have defined for $\beta>0$ the partition function $Z_\beta({\bf t})=Z_\beta( {\bf t}~|~L,N_\ell, \sigma^2,X_{N_0},Y_{N_0})$ by 
\begin{align}\label{eq:part-def-si}
 Z_\beta({\bf t})&:= A_\beta \int  \exp\left[-\sum_{\ell=1}^{L+1}\frac{N_{\ell-1}}{2\sigma^2}\big|\big|W^{(\ell)}\big|\big|_F^2 -\frac{\beta}{2}\big|\big|Y-\prod_{\ell=1}^{L+1} W^{(\ell)}X\big|\big|_2^2 -i\theta\cdot {\bf t} \right]\prod_{\ell=1}^{L+1}dW^{(\ell)}
\end{align}
and set
\[
Z_\infty( {\bf t}~|~L,N_\ell, \sigma^2,X_{N_0},Y_{N_0}):= \lim_{\beta\gives \infty}Z_\beta( {\bf t}~|~L,N_\ell, \sigma^2,X_{N_0},Y_{N_0}),
\]
where 
\[
A_{\beta} =\det(X_{N_0}^TX_{N_0})^{1/2}\lr{2\pi \beta}^{P/2}
\]
and $Y_{N_0,\perp}$ is the projection of $Y_{N_0}$ onto the orthogonal complement to the row span of $X_{N_0}$. The denominator $Z_\infty({\bf 0})$ is often called the Bayesian model evidence and represents the probability of the data $(X_{N_0},Y_{N_0})$ given the model (i.e. depth $L$, layer widths $N_1,\ldots, N_L$ and prior scale $\sigma^2$). Note that the constant $A_\beta$ cancels in the ratio \eqref{eq:post-def-si} and in any computations involving maximizing ratios of model evidence. 

The effective depth of the prior is measured by $L/N$; more precisely, by
\begin{equation}
\lpre =\lpre(N_1,\ldots, N_L) =\text{ effective depth of prior }:= \sum_{\ell=1}^L \frac{1}{N_\ell}.
\end{equation}
Beyond the justification of effective rank provided in the main text, we offer related intuition here to justify calling $\lpre$ the effective depth. Aside from the simple multiplicative dependence on $\sigma^2$, the norm of $\theta$ under the prior depends at large $N,L$ only on $\lpre$. For instance, a simple computation (Equation (9) in \cite{hanin2020products}) shows that with $\sigma^2 =1 $
\[
\lim_{\substack{N_1+\cdots+N_L\gives \infty \\ \lpre(N_1,\ldots, N_L)\gives \lambda}} \log \norm{\theta}^2 = \mN\lr{-\frac{\lambda}{2},\lambda},
\]
where the convergence is in distribution. Due to the rotational invariance $\theta$, we thus see that with $\sigma^2 = 1$, it is $\lpre$, as opposed to $L$, that gives a full description of the prior. Moreover, taking $\lpre\gives 0$ (even if $L\gives \infty$) gives the same (Gaussian) prior over predictions $\theta^Tx$ as one would obtain by simply starting with $L=0$.

As argued in the main text by the Poincare-Borel lemma, the posterior in the large-data limit will always have the form of a normal distribution. We formalize this statement here, taking $\theta_{*,N_0}$ to be the minimum-norm interpolant
\begin{equation}
\theta_{*,N_0}:=\argmin_{\theta\in \R^{N_0}}\norm{\theta}_2\quad \text{s.t.} \quad \theta^T X_{N_0} = Y_{N_0}.
\end{equation}

\begin{lemma}[Asymptotic Normality of Posterior]
\label{lem:normal}
For each $N_0$, consider a collection of $k\geq 1$ test points
\[
{\bf x}_{N_0} = \lr{x_{j;N_0},\,j=1,\ldots,k}.
\]
Suppose that 
\begin{itemize}
    \item For each $\alpha_0\in (0,1)$ there exists a vector $\mu_{*}=\lr{\mu_{*,j},\, j=1,\ldots,  k}\in \R^k$ such that
    \begin{equation}\label{eq:mu-def}
    \lim_{\substack{P,N_0\gives \infty \\ P/N_0\gives \alpha_0\in (0,1)}}\inprod{\theta_{*,N_0}}{x_{j,N_0}} =  \mu_{*,j}\qquad \text{ almost surely}.    
    \end{equation}
    \item For each $\alpha_0\in (0,1)$ there exists a positive semi-definite $k\times k$ matrix $\Sigma_\perp$ such that
    \begin{equation}\label{eq:Sig-def}
        \lim_{\substack{P,N_0\gives \infty \\ P/N_0\gives \alpha_0\in (0,1)}} \lr{\frac{1}{N_0-P}\inprod{x_{i,N_0}^\perp}{x_{j,N_0}^\perp}}_{1\leq i,j\leq k} = \Sigma_\perp\qquad \text{ almost surely},
    \end{equation}
    where $x_{j,N_0}^\perp$ denotes the projection of $x_j$ onto the orthogonal complement of the column space of $X_{N_0}$. 
\end{itemize}
Assume that the convergence of the posterior for $\norm{\theta_\perp}^2$ in the large-$N$ limit satisfies
\begin{align*}
\norm{\theta_\perp}^2 \to \nu(1-\alpha_0)c
\end{align*}
for some constant $c$. Then the distribution over posterior predictions evaluated on ${\bf x}_{N_0}$
\[
f({\bf x}_{N_0}) = \lr{a^Tx_{j,N_0},\, j=1,\ldots, k},\qquad \theta\sim \Papost
\]
converges weakly to a $k$-dimensional Gaussian: 
\begin{align}
\label{eq:normal-post}f({\bf x}_{N_0}) ~\gives~ \mathcal N\lr{\mu_*, ~ \nu c\Sigma_{\perp}}.
\end{align}
\end{lemma}

\subsection{Proof of Theorem
\ref{thm:Z-form}
%\ref{thm:Z-form}
}\label{sec:Z-form-pf}
Recall that, by definition, 
\[
Z_\beta({\bf x}_{N_0},{\bf t})= A_{\beta} \Ep{  \exp\left[ -\frac{\beta}{2}\big|\big|Y-\prod_{\ell=1}^{L+1}W^{(\ell)}X\big|\big|_2^2 -i\inprod{ f({\bf x}_{N_0})}{{\bf t}}\right]} ,
\]
where
\[
{\bf x}_{N_0} = \lr{x_{j,N_0},\, j=1,\ldots, k}\subseteq \R^{N_0},\quad x_{j,N_0}\in \R^{N_0},
\]
the expectation is over $W_{ij}^{(\ell)}\sim \mathcal N(0,\sigma^2/N_{\ell-1})$ and
\begin{equation*}
A_{\beta} =\det(X^TX)^{1/2}\lr{2\pi \beta}^{P/2}.
%A_{\beta} =\exp\left[\frac{\beta}{2}\norm{Y_\perp}^2\right]\det(X^TX)^{1/2}\lr{2\pi \beta}^{P/2}.
\end{equation*}
The first step in proving Theorem \ref{thm:Z-form} is to write
\[
Y = \theta_*^T X ,%+ Y_{\perp},\qquad Y_\perp \in \col(X)^\perp,
\]
where here and throughout the proof we suppress the subscripts in $X_{N_0},Y_{N_0}, \theta_{*,N_0}$. This yields
\begin{equation}\label{eqsi:Z-form-1}
Z_\beta({\bf x}_{N_0},{\bf t})= A_\beta  \Ep{  \exp\left[ -\frac{\beta}{2}\big|\big|\theta_*^TX-\prod_{\ell=1}^{L+1}W^{(\ell)}X\big|\big|_2^2 -i\inprod{ f({\bf x}_{N_0})}{{\bf t}}\right]}.    
%Z_\beta({\bf x}_{N_0},{\bf t})= A_\beta \exp\left[-\frac{\beta}{2}\norm{Y_\perp}^2\right] \Ep{  \exp\left[ -\frac{\beta}{2}\big|\big|\theta_*^TX-\prod_{\ell=1}^{L+1}W^{(\ell)}X\big|\big|_2^2 -i\inprod{ f({\bf x}_{N_0})}{{\bf t}}\right]}.    
\end{equation}
Next, since $f(x)$ is a linear function of $x$, we have
\[
\inprod{ f({\bf x}_{N_0})}{{\bf t}} = f\lr{{\bf x}_{N_0}\cdot {\bf t}},
\]
and hence, suppressing the dependence on $N_0$, we will write
\[
Z_\beta({\bf x}_{N_0},{\bf t}) = Z_\beta(x,1)=:Z_\beta(x),\qquad x:={\bf x}_{N_0}\cdot {\bf t}.
\]
To prove \ref{thm:Z-form}, we first derive the following expression for $Z_\beta(x)$ for general $\beta$.
\begin{proposition}\label{prop:Z-gen}
For any $\beta > 0$, the partition function $Z_\beta(x)$ equals
\begin{align*}
&\prod_{\ell=1}^L\Gamma\lr{\frac{N_\ell}{2}}^{-1} \exp\left[-i\theta_*^Tx_{||}\right]\\ &\qquad \times \int_{\col{(X)}} d\zeta\exp\left[-\frac{\norm{X^{\dagger}(\tau - x_{||})}^2}{2\beta}+i\theta_*^T\zeta\right]~ \G{1,L}{L,1}{M\norm{\zeta}^2 +M\norm{x_\perp}^2}{1-\frac{{\bf N}}{2}}{0},
\end{align*}
where $X^{\dagger}$ is the pseudo-inverse of $X$ and 
\[
4M = \prod_{\ell=0}^L \frac{2\sigma^2}{N_\ell}.
\]
\end{proposition}
\begin{proof}
We begin the proof of Proposition \ref{prop:Z-gen} by integrating out the final layer weights $W^{(L+1)}$ and introducing a dual variable $t\in \R^{P}$, as in the following.

%%%%LEMMA%%%%
\begin{lemma}\label{L:partition-form}
Write
\[
X^L:=W^{(L)}\cdots W^{(1)} X,\qquad x^L =W^{(L)}\cdots W^{(1)} x.
\]
We have
\begin{equation}\label{E:Z-noa}
Z_\beta(x) =\det(X^TX)^{1/2}\int_{\R^P} \Ep{\exp\left[-\frac{\norm{t}_2^2}{2\beta}+i\theta_*^TXt-\frac{\sigma^2}{2N_L}\norm{X^Lt + x^L}^2\right] }dt,
\end{equation}
\end{lemma}
%%%%LEMMA%%%%

%%%%PROOF%%%%
\begin{proof}
From the following identity
\begin{align*}
   1= \int_{\R^P} \frac{dt}{(2\pi \beta)^{P/2}} \exp\left[-\frac{1}{2\beta} \norm{t^T- i\beta\lr{\theta_*^TX-W^{(L+1)}X^L}}^2\right] 
\end{align*}
we conclude
\begin{align*}
 & \exp\left[-\frac{\beta}{2}\norm{\theta_*^TX-W^{(L+1)}X^L}^2\right]\\
  &\quad =  \int \frac{d t}{(2\pi \beta)^{P/2}} \exp\left[-\frac{1}{2\beta} \norm{t^T- i\beta\lr{\theta_*^TX-W^{(L+1)} X^L}}^2-\frac{\beta}{2}\norm{\theta_*^TX-W^{(L+1)} X^L}^2\right] \\
  &\quad =\int \frac{d t}{(2\pi \beta)^{P/2}} \exp\left[-\frac{1}{2\beta} \norm{t}^2 +it\lr{\theta_*^TX-W^{(L+1)} X^L}\right].
\end{align*}
Substituting this into \eqref{eqsi:Z-form-1} we find
\begin{align*}
Z_\beta(x)&= \frac{A_{\beta}}{(2\pi \beta)^{P/2}} \int_{\R^P}\Ep{ \exp\left[i\lr{\theta_*^TXt-W^{(L+1)}\lr{X^Lt +x^L}} -\frac{\norm{t}^2}{2\beta}\right]} dt    
\end{align*}
Now we compute the expectation over the final layer weights $W^{(L+1)}$ by completing the square:
\begin{align*}
    -\frac{N_L}{2\sigma^2}\norm{W^{(L+1)}}_2^2 - iW^{(L+1)}\lr{X^Lt + x^L}= &-\frac{N_L}{2\sigma^2}\left[\norm{W^{(L+1)}-i\frac{\sigma^2}{N_L}(X^Lt+x^L)}_2^2\right]\\
    &- \frac{\sigma^2}{2N_L}\norm{X^Lt + x^L}_2^2.
\end{align*}
This yields
\[
Z_\beta (x) =\frac{A_{\beta}}{(2\pi \beta)^{P/2}}\int_{\R^P}\Ep{ \exp\left[ -\frac{\norm{t}^2}{2\beta}+i\theta_*^TXt-\frac{\sigma^2}{2N_L}\norm{X^Lt + x^L}^2\right]}dt,
%Z_\beta (x) =\frac{A_{\beta}}{(2\pi \beta)^{P/2}}\exp\left[-\frac{\beta}{2}\norm{Y_\perp}^2\right] \int_{\R^P}\Ep{ \exp\left[ -\frac{\norm{t}^2}{2\beta}+i\theta_*^TXt-\frac{\sigma^2}{2N_L}\norm{X^Lt + x^L}^2\right]}dt,
\]
 completing the proof.
\end{proof}
%%%%PROOF%%%%

\noindent For each fixed $t$, note that
\[
X^Lt + x^L = W^{(L)}\cdots W^{(1)}\lr{Xt + x}.
\]
Hence, the expectation 
\[
\Ep{ \exp\left[-\frac{\sigma^2}{2N_L}\norm{X^Lt + x^L}^2\right]}
\]
equals the Laplace transform 
\[
\mathcal L_{Q_{N,L}}\lr{M'\norm{Xt+x}^2}
\]
of the random variable
\begin{equation}\label{E:QNL-def}
Q_{N,L} = \norm{ \left\{\prod_{\ell=1}^L \widehat{W}^{(\ell)}\right\} u}^2,\qquad \widehat{W}^{(\ell)} \sim \mN\lr{0,I_{N_{\ell}\times N_{\ell-1}}}\text{ independent},    
\end{equation}
where $u$ is any unit vector (the distribution is the same for any $u$ since $W^1$ is rotationally invariant) evaluated at $\norm{Xt+x}^2$ times 
\[
M' :=\frac{1}{2} \lr{\prod_{\ell=1}^L \frac{\sigma^2}{N_\ell}} 
\]
Thus, we obtain 
\begin{equation}\label{E:Z-lap}
Z_\beta(x) =\det(X^TX)^{1/2} \int dt \exp\left[-\frac{\norm{t}^2}{2\beta}+i\theta_*^TXt\right] \mathcal L_{Q_{N,L}}\lr{M'\norm{Xt+x}^2},
\end{equation}
To proceed we rewrite the Laplace transform in the preceding line in terms of a Meijer-G function.

%%%%LEMMA%%%%
\begin{lemma}\label{L:QNL-form}
Let $Q_{N,L}$ be defined as in \eqref{E:QNL-def}. Then, 
\begin{align*}
\mathcal L_{Q_{N,L}}\lr{\tau}= \lr{\prod_{\ell=1}^L \Gamma \lr{\frac{N_\ell}{2}}}^{-1}~\G{1,L}{L,1}{2^L\tau}{1-\frac{N_L}{2},\dots,1-\frac{N_1}{2}}{0}.
\end{align*}
\end{lemma}
%%%%LEMMA%%%%
%%%%PROOF%%%%
\begin{proof}
The density of $Q_{N,L}^{1/2}$ is known from \cite{zavatone2021exact}:
\[
p_{Q_{N,L}^{1/2}}(\rho) = \frac{2^{1-L/2}}{\Gamma\lr{\frac{N_1}{2}}\cdots \Gamma\lr{\frac{N_{L}}{2}}} \G{L,0}{0,L}{\frac{\rho^2}{2^L  }}{-}{\frac{N_L-1}{2},\cdots, \frac{N_1-1}{2}}.
\]
Hence, the density of $\widehat{Q}_{N,L}$ is 
\begin{align*}
p_{Q_{N,L}}(\rho) &= \frac{1}{2\rho^{1/2}}  p_{Q_{N,L}}(\rho^{1/2})\\
&=\frac{2^{-L}}{\Gamma\lr{\frac{N_1}{2}}\cdots \Gamma\lr{\frac{N_{L}}{2}}} \lr{\frac{2^L}{\rho}}^{1/2}\G{L,0}{0,L}{\frac{\rho}{2^L  }}{-}{\frac{N_L-1}{2},\cdots, \frac{N_1-1}{2}}\\
&=\frac{2^{-L}}{\Gamma\lr{\frac{N_1}{2}}\cdots \Gamma\lr{\frac{N_{L}}{2}}} \G{L,0}{0,L}{\frac{\rho}{2^L  }}{-}{\frac{N_L}{2}-1,\cdots, \frac{N_1}{2}-1},
\end{align*}
where in the last step we've used \eqref{E:G-polyprod}. Hence, 
\begin{align*}
    \mathcal L_{Q_{N,L}}(\tau) &= \int_0^\infty e^{-\tau \rho}  \G{L,0}{0,L}{\frac{\rho}{2^L  }}{-}{\frac{N_L}{2}-1,\cdots, \frac{N_1}{2}-1} d\rho\\
    &=\lr{\Gamma\lr{\frac{N_1}{2}}\cdots \Gamma\lr{\frac{N_{L}}{2}}}^{-1} \frac{1}{2^L\tau} \G{L,1}{1,L}{\frac{1}{2^L \tau }}{0}{\frac{N_L}{2}-1,\cdots, \frac{N_1}{2}-1}\\
    &=\lr{\Gamma\lr{\frac{N_1}{2}}\cdots \Gamma\lr{\frac{N_{L}}{2}}}^{-1} \G{L,1}{1,L}{\frac{1}{2^L \tau }}{0}{\frac{N_L}{2},\cdots, \frac{N_1}{2}}\\
    &=\lr{\Gamma\lr{\frac{N_1}{2}}\cdots \Gamma\lr{\frac{N_{L}}{2}}}^{-1}  \G{1,L}{L,1}{2^L \tau }{1-\frac{N_L}{2},\cdots, 1-\frac{N_1}{2}}{0}
\end{align*}
where we've used \eqref{E:G-weighted-Laplace} and \eqref{E:G-inversion}.
\end{proof}
%%%%PROOF%%%%

\noindent Combining the preceding Lemma with \eqref{E:Z-lap} gives
\begin{align}
\label{eqsi:ztilde-old}
    Z(x) &= \frac{\det(X^TX)^{1/2}}{\prod_{\ell=1}^L\Gamma\lr{\frac{N_\ell}{2}}} \int_{\R^{P}} dt \exp\left[-\frac{\norm{t}^2}{2\beta}+i\theta_*^TXt\right]~ \G{1,L}{L,1}{M\|Xt + x\|^2}{1-\frac{N_L}{2},\dots,1-\frac{N_1}{2}}{0},
\end{align}
where 
\[
4M =  \prod_{\ell=0}^L \frac{2\sigma^2}{N_\ell}.
\]
To complete the proof of Proposition \ref{prop:Z-gen}, we write
%\[
%t = t_{||}+t_{\perp},\qquad t_{||}\in \ker(X),\quad t_{\perp}\in \ker(X)^\perp
%\]
%and
\[
x = x_{||}+x_{\perp},\qquad x_{||}\in \col(X),\quad x_{\perp}\in \col(X)^\perp.
\]
%Writing $\norm{t}^2=\norm{t_{||}}^2 + \norm{t_{\perp}}^2$ allows us to write $Z_\beta(x)$ as 
%\begin{align*}
%\frac{\det(X^TX)^{1/2}}{\prod_{\ell=1}^L\Gamma\lr{\frac{N_\ell}{2}}} \int_{\R^P} dt \exp\left[-\frac{\norm{t}^2}{2\beta}+i\theta_*^TXt\right]~ \G{1,L}{L,1}{M\|Xt + x\twiddle{t}\|^2}{1-\frac{{\bf N}}{2}}{0}.
%\end{align*}
%Note that $X$ gives an isomorphism from $\ker(X)^\perp$ to $\im{(X)}$.
Note that $X$ gives an isomorphism from $\R^P$ to $\im{(X)}$. Thus, we may change variables to write $Z_\beta(x)$ as
\begin{align*}
\prod_{\ell=1}^L\Gamma\lr{\frac{N_\ell}{2}}^{-1} \int_{\R^{P}} d\zeta\exp\left[-\frac{\norm{X^{\dagger}\zeta}^2}{2\beta}+i\theta_*^T\zeta\right]~ \G{1,L}{L,1}{M\|\zeta + x\|^2}{1-\frac{{\bf N}}{2}}{0},
\end{align*}
where $X^{\dagger}$ the pseudo-inverse of $X$. Finally, note that
\[
\norm{\zeta+x}^2 = \norm{\zeta + x_{||}}^2 + \norm{x_\perp}^2 .
\]
Hence, changing variables to $\tau = \zeta + x_{||}$ yields the following expression for $Z_\beta(x)$:
\begin{align*}
&\prod_{\ell=1}^L\Gamma\lr{\frac{N_\ell}{2}}^{-1} \exp\left[-i\theta_*^Tx_{||}\right]\times \int_{\col{(X)}} d\zeta\exp\left[-\frac{\norm{X^{\dagger}(\tau - x_{||})}^2}{2\beta}+i\theta_*^T\zeta\right]~ \G{1,L}{L,1}{M\norm{\zeta}^2 +M\norm{x_\perp}^2}{1-\frac{{\bf N}}{2}}{0}.
\end{align*}
This is precisely the statement of Proposition \ref{prop:Z-gen}. 
\end{proof}

\noindent By taking $\beta \gives \infty$ in Proposition \ref{prop:Z-gen}, we see that
\begin{equation}\label{eqsi:Z-form-2}
Z_\infty(x)=\frac{\exp\left[-i\theta_*^Tx_{||}\right]}{\prod_{\ell=1}^L\Gamma\lr{\frac{N_\ell}{2}}
}
 \int_{\R^P} \exp\left[i\theta_*^T\zeta\right]~ \G{1,L}{L,1}{M\norm{\zeta}^2 +M\norm{x_\perp}^2}{1-\frac{{\bf N}}{2}}{0}d\zeta.
\end{equation}
In order to simplify this expression further, we pass to polar coordinates
\begin{align*}
   &\int_{\R^P} \exp\left[i\theta_*^T\zeta\right]~ \G{1,L}{L,1}{M\norm{\zeta}^2 +M\norm{x_\perp}^2}{1-\frac{{\bf N}}{2}}{0}d\zeta \\
   &\qquad =\int_{0}^\infty \rho^{P-1}\left\{ \int_{S^{P-1}}\exp\left[i\rho \theta_*^T\theta\right]d\theta\right\}~ \G{1,L}{L,1}{M\rho^2 +M\norm{x_\perp}^2}{1-\frac{{\bf N}}{2}}{0}d\rho.
\end{align*}
By the definition of the Bessel function and the relation \eqref{E:G-BesselJ}, we have
\begin{align*}
\int_{S^{P-1}}\exp\left[i\rho \theta_*^T\theta\right]d\theta&= (2\pi)^{P/2}\lr{\rho \norm{\theta_*}}^{-\frac{P-2}{2}} J_{\frac{P-2}{2}}\lr{\rho \norm{\theta_*}}\\    
&= (2\pi)^{P/2}\lr{\rho^2 \norm{\theta_*}^2}^{-\frac{P-2}{4}} \G{1,0}{0,2}{\frac{\rho^2\norm{\theta_*}^2}{4}}{-}{\frac{P-2}{4},-\frac{P-2}{4}}\\  
&=2\pi^{P/2} \G{1,0}{0,2}{\frac{\rho^2\norm{\theta_*}^2}{4}}{-}{0,-\frac{P-2}{2}}.  
\end{align*}
We therefore obtain 
\begin{align*}
   &\int_{\R^P} \exp\left[i\theta_*^T\zeta\right]~ \G{1,L}{L,1}{M\norm{\zeta}^2 +M\norm{x_\perp}^2}{1-\frac{{\bf N}}{2}}{0}d\zeta \\
   &\qquad =\pi^{P/2}\int_{0}^\infty \rho^{\frac{P-2}{2}} \G{1,0}{0,2}{\frac{\rho\norm{\theta_*}^2}{4}}{-}{0,-\frac{P-2}{2}}~ \G{1,L}{L,1}{M\rho +M\norm{x_\perp}^2}{1-\frac{{\bf N}}{2}}{0}d\rho\\
   &\qquad =\lr{\frac{4}{\norm{\theta_*}^2}}^{\frac{P-2}{2}}\pi^{P/2}\int_{0}^\infty  \G{1,0}{0,2}{\frac{\rho\norm{\theta_*}^2}{4}}{-}{\frac{P-2}{2},0}~ \G{1,L}{L,1}{M\rho +M\norm{x_\perp}^2}{1-\frac{{\bf N}}{2}}{0}d\rho\\
   &\qquad =\lr{\frac{4}{\norm{\theta_*}^2}}^{\frac{P}{2}}\pi^{P/2}\frac{\norm{\theta_*}^2}{4M}\int_{0}^\infty  \G{1,0}{0,2}{\frac{\rho\norm{\theta_*}^2}{4M}}{-}{\frac{P-2}{2},0}~ \G{1,L}{L,1}{\rho +M\norm{x_\perp}^2}{1-\frac{{\bf N}}{2}}{0}d\rho.
\end{align*}
We now apply \eqref{E:G-series} to find
\begin{align*}
    &\int_{0}^\infty  \G{1,0}{0,2}{\frac{\rho\norm{\theta_*}^2}{4M}}{-}{\frac{P-2}{2},0}~ \G{1,L}{L,1}{\rho +M\norm{x_\perp}^2}{1-\frac{{\bf N}}{2}}{0}d\rho\\
    &\qquad = \sum_{k=0}^\infty \frac{1}{k!}\lr{-M\norm{x_\perp}^2}^k \G{L+1,0}{0,L+1}{\frac{\norm{\theta_*}^2}{4M}}{-}{\frac{P-2}{2}, \frac{{\bf N}}{2}+k-1}.
\end{align*}
Therefore, 
\begin{align*}
   &\int_{\R^P} \exp\left[i\theta_*^T\zeta\right]~ \G{1,L}{L,1}{M\norm{\zeta}^2 +M\norm{x_\perp}^2}{1-\frac{{\bf N}}{2}}{0}d\zeta \\
   &\qquad =\lr{\frac{4}{\norm{\theta_*}^2}}^{\frac{P}{2}}\pi^{P/2} \sum_{k=0}^\infty \frac{1}{k!}\lr{-M\norm{x_\perp}^2}^k \G{L+1,0}{0,L+1}{\frac{\norm{\theta_*}^2}{4M}}{-}{\frac{P}{2}, \frac{{\bf N}}{2}+k}.
\end{align*}
Putting this all together yields
\begin{align*}
    Z_\infty(x) = \lr{\frac{4\pi}{\norm{\theta_*}^2}}^{\frac{P}{2}}\prod_{\ell=1}^L \Gamma\lr{\frac{N_\ell}{2}}^{-1}\sum_{k=0}^\infty \frac{1}{k!}\lr{-M\norm{x_\perp}^2}^k \G{L+1,0}{0,L+1}{\frac{\norm{\theta_*}^2}{4M}}{-}{\frac{P}{2}, \frac{{\bf N}}{2}+k},
\end{align*}
completing the proof. 
\hfill $\square$

\subsection{Proof of Theorem 
\ref{thm:logG}
%\ref{thm:logG}
}
\label{pf:logG}
In this section, we derive several apparently novel asymptotic expansion of the Meijer-$G$ functions of the form 
\begin{equation}\label{eq:G-fn}
\G{L+1,0}{0,L+1}{\frac{\norm{\theta_*}^2}{4M}}{-}{\frac{P}{2},\frac{{\bf N}}{2}+k} := \G{L+1,0}{0,L+1}{\frac{\norm{\theta_*}^2}{4M}}{-}{\frac{P}{2},\frac{N_1}{2}+k,\ldots,\frac{N_L}{2}+k}.
\end{equation}

Our first step is to obtain a contour integral representation of the Meijer-G functions we are studying. To state the exact result, consider the following independent $\Gamma$ random variables:
\begin{align}
 \label{eq:phi-def}   \phi_j \sim \begin{cases}\Gamma\left(\frac{N_j}{2}+k+1, \frac{2\sigma^2}{N_j}\right), &j=1,\dots,L\\ \Gamma\left(\frac{P}{2}+1, \frac{2\sigma^2}{P}\frac{\alpha_0}{\norm{\theta_*}^2}\right), &j=0\end{cases}.
\end{align}
As we recalled in \S \ref{sec:gammarv}, the moments of $\phi_j$ can be explicitly written in terms of $\Gamma$ functions. Moreover, the Meijer-G functions \eqref{eq:G-fn} can be interpreted, up to a scaling factor, as densities of the product of products of $\phi_j$'s. This is allows us to obtain the following
\begin{lemma}\label{lem:contour}
Fix $L, N, N_0,\ldots, N_L\geq 1$ as well as $\norm{\theta_*}>0$ and define $M$ by
\begin{align}\label{eq:M-def}
4M := \prod_{\ell=0}^L \frac{2\sigma^2}{N_\ell}.
\end{align}
For any $N\geq 1$ we have 
\begin{align}
\notag \mathrm{Den}_{\phi_0\cdots \phi_{L}}(1) &= \frac{\norm{\theta_*}^2}{4M}\G{L+1,0}{0,L+1}{\frac{\norm{\theta_*}^2}{4M}}{-}{\frac{P}{2},\frac{{\bf N}}{2}+k}\left[\Gamma\left(\frac{P}{2}+1\right)\right]^{-1}\prod_{\ell=1}^L\left[\Gamma\left(\frac{N_\ell}{2}+k+1\right)\right]^{-1}\\
\label{eq:G-contour}&= \frac{1}{2\pi}\int_{\mathcal C} \exp\left[\Phi(z)\right]dz,
\end{align}
where $\mathcal C\subseteq\C$ is the contour that runs along the real line from $-\infty$ to $\infty$ and
\begin{align}
\notag   \Phi(z)&= -iz\log\lr{\frac{2\sigma^2}{P}\frac{\alpha_0}{\norm{\theta_*}^2}}+\log\lr{\frac{\Gamma\lr{\frac{P}{2}+1-iz}}{\Gamma\lr{\frac{P}{2}+1}}}\\
   &+\sum_{\ell=1}^L\left\{-iz\log\lr{\frac{2\sigma^2}{N_\ell}}+\log\lr{\frac{\Gamma\lr{\frac{N_\ell}{2}+k+1-iz}}{\Gamma\lr{\frac{N_\ell}{2}+k+1}}}    \right\}.
% \notag   \Psi(z)&= -iz\log\lr{\frac{2\sigma^2}{P}\frac{\alpha_0}{\norm{\theta_*}^2}}+\frac{1}{N}\log\lr{\frac{\Gamma\lr{\frac{P}{2}+1-iNz}}{\Gamma\lr{\frac{P}{2}+1}}}\\
 % \label{E:psi-def}  &+\sum_{\ell=1}^L\left\{-iz\log\lr{\frac{2\sigma^2}{N_\ell}}+\frac{1}{N}\log\lr{\frac{\Gamma\lr{\frac{N_\ell}{2}+k+1-iNz}}{\Gamma\lr{\frac{N_\ell}{2}+k+1}}}    \right\}.
\end{align}
\end{lemma}
\begin{proof}
We use the relationship \eqref{eq:G-gamma} between $G$ functions and densities of products of Gamma random variables to write
\begin{align*}
\prod_{\ell=1}^L\left[\Gamma\left(\frac{N_\ell}{2}+k+1\right)\right]^{-1}\left[\Gamma\left(\frac{P}{2}+1\right)\right]^{-1}\frac{\norm{\theta_*}^2}{4M}\G{L+1,0}{0,L+1}{\frac{\norm{\theta_*}^2}{4M}}{-}{\frac{P}{2},\frac{{\bf N}}{2}+k} &= \mathrm{Den}_{\phi_0\cdots \phi_{L}}(1).
\end{align*}
For any positive random variable $X$ with density $d\mathbb P_X$, we have
\[
d\mathbb P_X(t) = t^{-1}d\mathbb P_{\log(X)}(\log(t)).
\]
Hence,
\begin{align*}
\mathrm{Den}_{\phi_0\cdots \phi_{L}}(1) &= \mathrm{Den}_{\sum_{\ell=0}^L\log \phi_\ell}(0)
\end{align*}
is proportional to the density of a sum of independent random variables evaluated at $\log(1)=0$. Further, recalling \eqref{E:gamma-MGF}, we find by Fourier inversion that
\[
\frac{1}{2\pi}\int_{-\infty}^\infty \exp\left[\Phi(t)\right] dt
\]
equals
\begin{align*}
\prod_{\ell=1}^L\left[\Gamma\left(\frac{N_\ell}{2}+k+1\right)\right]^{-1}\left[\Gamma\left(\frac{P}{2}+1\right)\right]^{-1}\frac{\norm{\theta_*}^2}{4M}\G{L+1,0}{0,L+1}{\frac{\norm{\theta_*}^2}{4M}}{-}{\frac{P}{2},\frac{{\bf N}}{2}+k},
\end{align*}
where
\begin{align}
 \notag   \Phi(t)&= -it\log 1-it\log\lr{\frac{2\sigma^2}{P}\frac{\alpha_0}{\norm{\theta_*}^2}}+\log\lr{\frac{\Gamma\lr{\frac{P}{2}+1-it}}{\Gamma\lr{\frac{P}{2}+1}}}\\
    &+\sum_{\ell=1}^L\left\{-it\log\lr{\frac{2\sigma^2}{N_\ell}}+\log\lr{\frac{\Gamma\lr{\frac{N_\ell}{2}+k+1-it}}{\Gamma\lr{\frac{N_\ell}{2}+k+1}}}    \right\}.
\end{align}
%Making the change of variables $t\mapsto Nt$ completes the proof. 
\end{proof}
Note that \eqref{eq:G-contour-1} expresses $G$ as a contour integral of a meromorphic function $\exp(N\Psi)$ with poles at 
\[
-\frac{m}{2}-k-1 -\nu,\qquad \nu \in \N,\, m\in \set{P,N_1,\ldots, N_L}.
\]
Our goal is to evaluate the contour integral representation \eqref{eq:G-contour} for the Meijer-G function using the Laplace method. To do so, we will use the following standard procedure:
\begin{enumerate}
    \item Contour deformation: $\mathcal C$ into a union of several contours $\mathcal C_1\cup\cdots  \cup \mathcal C_R$. Contours $\mathcal C_1,\mathcal C_R$ are non-compact, whereas the contours $\mathcal C_2,\ldots, \mathcal C_{R-1}$ do not extend to infinity. We will need to choose the contours $\mathcal C_2,\ldots, \mathcal C_{R-1}$ so that exactly one of them passes through what will turn out to be the dominant critical point $\zeta_*$ of $\Psi$ and does so in the direction of steepest descent. Moreover, on the contour $\mathcal C_2$, the imaginary part of the phase will be constant (in fact equal to $0$). 
    \item Localization to compact domain of integration: Show that the integrand $\exp\left[N\Psi(z)\right]$ is integrable and exponentially small in $N$ on the contours $\mathcal C_1,\mathcal C_R$. In particular, for any $K>0$, we will find that modulo errors of size $O(e^{-KN})$ we may therefore focus on the integral over $\mathcal C_2,\ldots, \mathcal C_{R-1}$. The ability to choose $K$ will be important since the entire integral is exponentially small in $N$.
    \item Computing derivatives of $\Psi$ at $\zeta_*$: Now that we have reduced the integral \eqref{eq:G-contour} to a compact domain of integration it remains only to check that the critical point is non-degenerate, to compute $\Psi(\zeta_*), \frac{d}{dz}\Psi(\zeta_*),\frac{d^2}{dz^2}\Psi(\zeta_*)$, and to apply the Laplace method.
\end{enumerate}
We now proceed to give the details in the case when 
\[
L\text{ is fixed},\quad  N:=\min\set{N_0, P, N_\ell}\gives \infty,\quad \frac{P}{N_0}\gives \alpha_0\in (0,1),\quad \frac{P}{N}\gives \alpha\in (0,\infty).
\]
This is a generalization of case (a) of Theorem \ref{thm:logG}. To proceed, we make the change of variables $t\mapsto Nt$ in \eqref{eq:G-contour} to get that 
\begin{align}
\frac{\norm{\theta_*}^2}{4M}\G{L+1,0}{0,L+1}{\frac{\norm{\theta_*}^2}{4M}}{-}{\frac{P}{2},\frac{{\bf N}}{2}+k}\left[\Gamma\left(\frac{P}{2}+1\right)\right]^{-1}\prod_{\ell=1}^L\left[\Gamma\left(\frac{N_\ell}{2}+k+1\right)\right]^{-1}= \frac{N}{2\pi}\int_{\mathcal C} \exp\left[N\Psi(z)\right]dz,
\end{align}
where
\begin{align}
    \notag   \Psi(z)&= -iz\log\lr{\frac{2\sigma^2}{P}\frac{\alpha_0}{\norm{\theta_*}^2}}+\frac{1}{N}\log\lr{\frac{\Gamma\lr{\frac{P}{2}+1-iNz}}{\Gamma\lr{\frac{P}{2}+1}}}\\
  \label{E:psi-def}  &+\sum_{\ell=1}^L\left\{-iz\log\lr{\frac{2\sigma^2}{N_\ell}}+\frac{1}{N}\log\lr{\frac{\Gamma\lr{\frac{N_\ell}{2}+k+1-iNz}}{\Gamma\lr{\frac{N_\ell}{2}+k+1}}}    \right\}.
\end{align}
With this rescaling, we now deform the contour of integration as follows by fixing constants $\delta,T>0$ (the value of $T$ will be determined by Lemma \ref{lem:Psi-est} and the value of $\delta$ will be determined by \eqref{eq:zeta-star2-def} and the sentence directly after it) and deforming the contour $\mathcal C$ as follows 
\[
\mathcal C\quad \mapsto \quad \bigcup_{j=0}^4\mathcal C_j,
\]
where
\begin{align*}
    \mathcal C_0&=\mathcal C_0(T)=(-\infty,-T]\\
    \mathcal C_1&=\mathcal C_1(T,C_\delta)=\text{linear interpolation from }-T\in \R\text{ to }-iC_\delta\in i\R\\
    \mathcal C_2&= \mathcal C_2(C_\delta)=[-iC_\delta,iC_\delta]\\
    \mathcal C_3&=\mathcal C_3(T,C_\delta)=\text{linear interpolation from }iC_\delta\in i\R\text{ to }T\in \R\\
    \mathcal C_4&=\mathcal C_4(T)=[T,\infty),
\end{align*}
where
\[
C_\delta := \delta -\min\set{\frac{P}{2N}, \frac{N_1}{2N}, \ldots,\frac{N_L}{2N}}.
\]
\begin{figure}
    \centering
    \includegraphics[scale=.2]{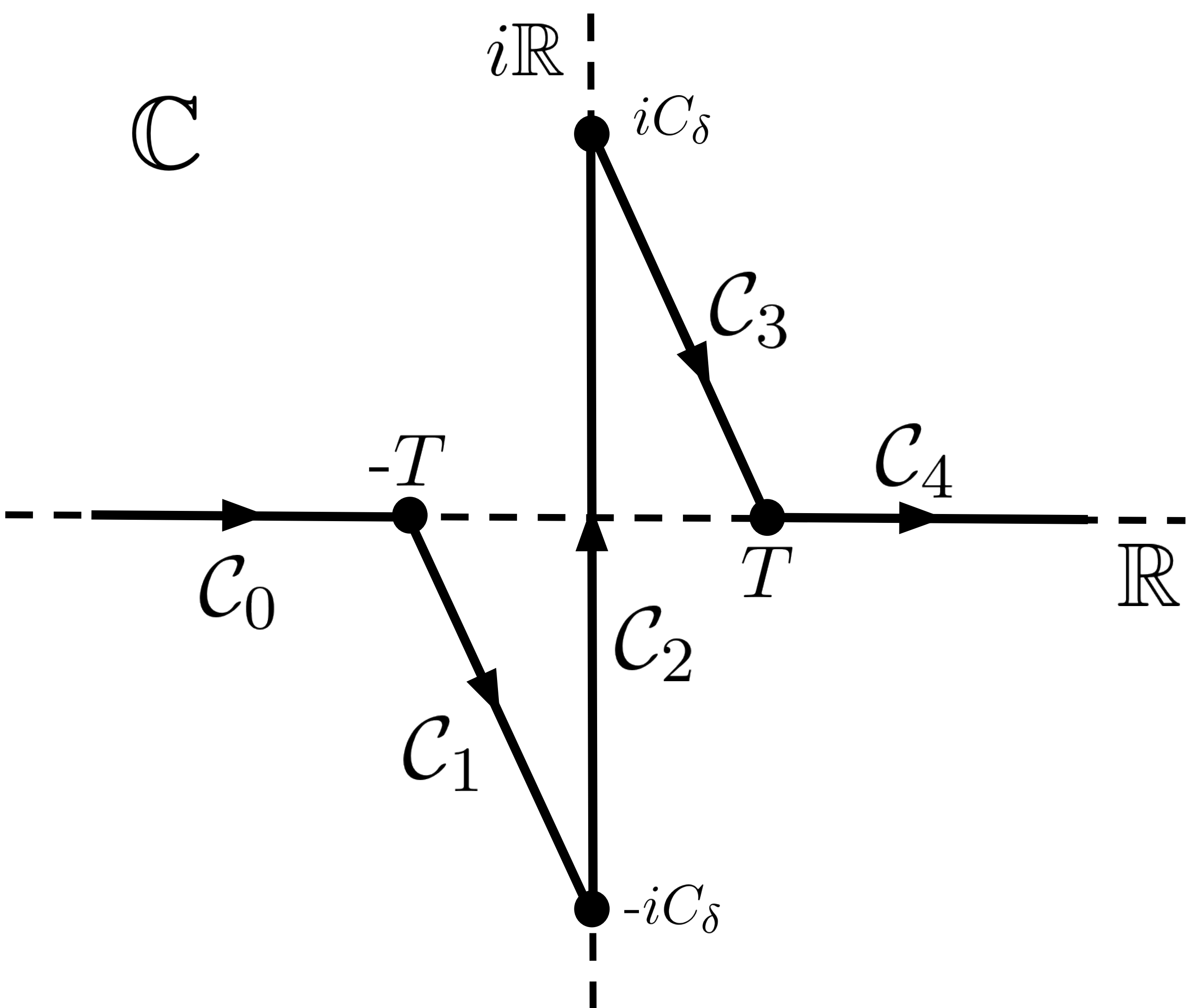}
    \caption{Deformed contour of integration for the proof of Theorem \ref{thm:logG}.}
    \label{fig:C-contour}
\end{figure}
See Figure \ref{fig:C-contour}. In order to evaluate the integral over each $\mathcal C_j$, it will be convenient to introduce
\[
\Psi(z)=\sum_{\ell=0}^L\Psi_\ell(z)
\]
where
\begin{equation}\label{eq:Psi-ell-def}
    \Psi_\ell(z)=\begin{cases}
    -iz\log\lr{\frac{2\sigma^2}{P}\frac{\alpha_0}{\norm{\theta_*}^2}}+\frac{1}{N}\log\lr{\frac{ \Gamma\lr{\frac{P}{2}+1-iNz}}{ \Gamma\lr{\frac{P}{2}+1}}},&\quad \ell = 0\\
    -iz\log\lr{\frac{2\sigma^2}{N_\ell}}+\frac{1}{N}\log\lr{\frac{ \Gamma\lr{\frac{N_\ell}{2}+k+1-iNz}}{ \Gamma\lr{\frac{N_\ell}{2}+k+1}}},&\quad \ell = 1,\ldots, L
    \end{cases}.
\end{equation}
The following Lemma allows us to throw away the contribution to \eqref{eq:G-contour-1} coming from $\mathcal C_0,\mathcal C_4$.

%%%%%%%%%%%%%%%%%%%%
%%%%%%%%%%%%%%%%%%%%
%%%%%% LEMMA %%%%%%%
%%%%%%%%%%%%%%%%%%%%
%%%%%%%%%%%%%%%%%%%%
\begin{lemma}\label{lem:Psi-est}
There exist $c,T_0>0$ such that
\begin{equation}\label{eq:Psi-est-1}
    \sup_{\substack{\abs{z}>T\\ z\in \R}} \frac{\Re \Psi(z)}{1+\abs{z}} \leq - c,\qquad \forall T\geq T_0.
\end{equation}
Moreover, for any $T,\delta>0$, writing
\begin{equation}\label{eq:S-def}
    S_{N,\delta,T}:=\set{z\in \C~|~ \abs{z}< T, \Im(z) >C_\delta},
\end{equation}
there exists $C>0$ such that
\begin{equation}\label{eq:Psi-est-2}
    \sup_{z\in S_{N,\delta, T}}\max\set{\abs{\Psi(z)},\abs{ \frac{d}{dz}\Psi(z)}} \leq C.
\end{equation}
\end{lemma}
%%%%%%%%%%%%%%%%%%%%
%%%%%%%%%%%%%%%%%%%%
%%%%%% LEMMA %%%%%%%
%%%%%%%%%%%%%%%%%%%%
%%%%%%%%%%%%%%%%%%%%

%%%%%%%%%%%%%%%%%%%%
%%%%%%%%%%%%%%%%%%%%
%%%%%% PROOF %%%%%%%
%%%%%%%%%%%%%%%%%%%%
%%%%%%%%%%%%%%%%%%%%
\begin{proof}
To show both \eqref{eq:Psi-est-1} and \eqref{eq:Psi-est-2} we will use the asymptotic expansion
\begin{equation}\label{eq:Gamma-exp-pf}
\log \Gamma\lr{z} \sim \lr{z-\frac{1}{2}}\log(z) - z + \text{const}+O\lr{\abs{z}^{-1}},\quad \text{as }\abs{z}\gives \infty,    
\end{equation}
which holds uniformly on sets of the form $\abs{z}< \pi - \epsilon$ for a fixed $\epsilon >0$. With $\Psi_\ell$ defined in \eqref{eq:Psi-ell-def}, we find for each $\ell=1,\ldots, L$ that uniformly over $z\in \R$ 
\begin{align}
  \notag  \Psi_\ell(z)&=-iz\log\lr{\frac{2\sigma^2}{N_\ell}} +\frac{1}{N}\lr{\log \Gamma\lr{\frac{N_\ell}{2}+k+1-iNz}-\log \Gamma\lr{\frac{N_\ell}{2}+k+1}}\\
  \notag  &=-iz\left[-1+\log\lr{\frac{2\sigma^2}{N_\ell}} + \log\lr{\frac{N_\ell}{2}+k+1-iNz}\right]\\
\notag    &\qquad +\frac{1}{N}\lr{\frac{N_\ell+1}{2}+k+1}\log\lr{1-\frac{iNz}{\frac{N_\ell}{2}+k+1}} + O(N^{-1})\\
\notag    &=-iz\left[-1+\log(\sigma^2) + \log\lr{1-\frac{2i N z}{N_\ell}}\right] + \frac{N_\ell}{2N}\log\lr{1-\frac{2iNz}{N_\ell}\cdot\frac{1}{1+\frac{2(k+1)}{N_\ell}}}+O(N^{-1})\\
\notag   &=-iz\left[-1+\log(\sigma^2)+\log\lr{1-\frac{2i N z}{N_\ell}}\right] + \frac{N_\ell}{2N}\log\lr{1-\frac{2iNz}{N_\ell}}+O(N^{-1}).
\end{align}
Obtaining a similar expression for $\ell=0$ and summing over $\ell$ proves \eqref{eq:Psi-est-2}. Moreover, for $\ell=1,\ldots, L$ we obtain uniformly on $z\in \R$ 
\begin{align*}
    \Re\Psi_\ell(z) = z\arg\lr{1-\frac{2iNz}{N_\ell}} + \frac{N_\ell}{2N}\log\abs{1-\frac{2iNz}{N_\ell}}+O(N^{-1}).
\end{align*}
Note that for $T_0$ sufficiently large there we have
\begin{equation}\label{eq:RePsi-bound}
\abs{z}> T_0\quad \Rightarrow \quad \arg\lr{1-\frac{2iNz}{N_\ell}} \sgn{z} \leq  -\frac{\pi}{4}\quad  \Rightarrow \quad \Re\Psi_\ell(z) \leq -\frac{\pi}{8}\abs{z}.
\end{equation}
A similar analysis applies to $\ell=0$ and completes the proof of \eqref{eq:Psi-est-1}. 
\end{proof}

%%%%%%%%%%%%%%%%%%%%
%%%%%%%%%%%%%%%%%%%%
%%%%%% PROOF %%%%%%%
%%%%%%%%%%%%%%%%%%%%
%%%%%%%%%%%%%%%%%%%%

\noindent Estimate \eqref{eq:Psi-est-1} in the previous Lemma shows that, for any $K,\delta >0$ there exists $T>0$ such that
\begin{equation}\label{eq:G-contour-1}
\frac{N}{2\pi}\int_{\mathcal C} \exp\left[N\Psi(z)\right] dz = \frac{N}{2\pi}\int_{\mathcal C_1(T,C_\delta)\cup \mathcal C_2(C_\delta) \cup \mathcal C_3(T,C_\delta)} \exp\left[N\Psi(z)\right] dz  + O(e^{-KN}).    
\end{equation}
To evaluate the right hand size of \eqref{eq:G-contour-1}, we derive the following uniform asymptotic expansion for $\Psi(z)$.
\begin{lemma}
Fix $\delta, T>0$ and define $S_{N,\delta, T}$ as in \eqref{eq:S-def}. Then, uniformly over $S_{N,\delta, T}$, we have
\[
\Psi(z) = \widehat{\Psi}(z)+\frac{1}{N}\twiddle{\Psi}(z)+O(N^{-2}),
\]
where
\begin{align}
 \label{eq:Psihat-def}   \widehat{\Psi}(z):&=iz\left[\log\lr{\frac{\norm{\theta_*}^2}{\sigma^{2(L+1)}\alpha_0}} + L + 1 - \log\lr{1-\frac{2iNz}{P}} - \sum_{\ell=1}^L \log\lr{1-\frac{2iNz}{N_\ell}}\right]\\
\notag    &+\frac{P}{2N}\log\lr{1-\frac{2iNz}{P}} +\sum_{\ell=1}^L \frac{N_\ell}{2N} \log\lr{1-\frac{2iNz}{N_\ell}}\\
\label{eq:Psitilde-def}    \twiddle{\Psi}(z):&=\frac{1}{2}\log\lr{1-\frac{2iNz}{P}} + \lr{k+\frac{1}{2}}\sum_{\ell=1}^L \log\lr{1-\frac{2iNz}{N_\ell}}.
\end{align}
\end{lemma}
\begin{proof}
This follows directly from expanding $\Psi_\ell(z)$ using \eqref{eq:Gamma-exp} for each $\ell=0,\ldots, L.$
\end{proof}
%%%%%%%%%%%%%%%%%%%%
%%%%%%%%%%%%%%%%%%%%
%%%% END PROOF %%%%%
%%%%%%%%%%%%%%%%%%%%
%%%%%%%%%%%%%%%%%%%%

Combining the preceding Lemma with \eqref{eq:G-contour-1} shows that for any $K>0$ there exists $T,\delta>0$ so that 
\begin{align}\label{eq:G-contour-2}
\notag &\frac{N}{2\pi}\int_{\mathcal C} \exp\left[N\Psi(z)\right] dz\\
&\qquad = \frac{N}{2\pi}\int_{\mathcal C_1(T,C_\delta)\cup \mathcal C_2(C_\delta) \cup \mathcal C_3(T,C_\delta)} \exp\left[N\widehat{\Psi}(z)+\twiddle{\Psi}(z)\right] dz\lr{1 + O(N^{-1})} + O(e^{-KN}) 
\end{align}
Moreover, differentiating \eqref{eq:Psihat-def} yields
\begin{align*}
\frac{d}{dz}\widehat{\Psi}(z) &= i\left[\log\lr{\frac{\norm{\theta_*}^2}{\sigma^{2(L+1)}\alpha_0}} - \log\lr{1-\frac{2iNz}{P}}-\sum_{\ell=1}^L \log\lr{1-\frac{2iNz}{N_\ell}}\right].
\end{align*}
Computing the real part of both sides show that
\[
\Re\lr{\frac{d}{dz}\widehat{\Psi}(z)}   = \arg\lr{1-\frac{2iNz}{P}}+\sum_{\ell=1}^L \arg\lr{1-\frac{2iNz}{N_\ell}}.
\]
Since all the arguments have the same sign, we conclude that the real part of $\frac{d}{dz}\widehat{\Psi}$ vanishes only when $z = i\zeta$ for $\zeta\in \R$. Further, 
\[
\Im\lr{\frac{d}{d\zeta}\widehat{\Psi}(i\zeta)}   =-\log\lr{\frac{\norm{\theta_*}^2}{\sigma^{2(L+1)}\alpha_0}} +\log\abs{1+\frac{2N\zeta}{P}}+\sum_{\ell=1}^L \log\abs{1+\frac{2N\zeta}{N_\ell}},
\]
which vanishes on $\mathcal C_2$ if and only if $\zeta = \zeta_*$ is the unique solution to \begin{equation}\label{eq:Psi-crit-2}
\frac{\norm{\theta_*}^2}{\sigma^{2(L+1)}\alpha_0} = \lr{1+\frac{2N\zeta_*}{P}}\prod_{\ell=1}^L \lr{1+\frac{2N\zeta_*}{N_\ell}}.    
\end{equation}
Indeed, observe that the right hand side of the equation on the preceding line increases monotonically in $\zeta_*$ from $-\infty$ to $+\infty$ as $\zeta_*$ varies in $(-\min\set{\frac{P}{2N},\frac{N_1}{2N},\ldots, \frac{N_L}{2N}},\infty)$.
Computing the correction to the saddle point, we differentiate
\begin{align*}
    \Im\lr{\frac{d}{d\zeta}\lr{\widehat{\Psi}(i\zeta) + \frac{1}{N}\twiddle{\Psi}(i\zeta)}} &= -\log\lr{\frac{\norm{\theta_*}^2}{\sigma^{2(L+1)}\alpha_0}} +\log\abs{1+\frac{2N\zeta}{P}}+\sum_{\ell=1}^L \log\abs{1+\frac{2N\zeta}{N_\ell}} \\
    &\qquad + \frac{1}{N}\left[\frac{1}{2\zeta + P/N} + \sum_{\ell=1}^L\frac{N}{N_\ell}\frac{1+2k}{1+2\zeta N_{\ell}/N}\right]
\end{align*}
and set the derivative to zero at $\zeta = \zeta_* + \zeta_{**}/N$, giving
\begin{align*}
    0 = \frac{1}{N}\left[\frac{2\zeta_{**}}{2\zeta_*+P/N} + \frac{1}{2\zeta_*+P/N} + \sum_{\ell=1}^L \frac{2\zeta_{**}}{2\zeta_*+N_\ell/N} + \frac{N}{N_\ell}\frac{1+2k}{1+2\zeta_* N_{\ell}/N}\right].
\end{align*}
This is solved by
\begin{align}\label{eq:zeta-star2-def}
    \zeta_{**} &= -\frac{1}{2}\frac{[2\zeta_*+P/N]^{-1} + (1+2k)\sum_\ell[2\zeta_* + N_\ell/N]^{-1}}{[2\zeta_*+P/N]^{-1} + \sum_\ell[2\zeta_* + N_\ell/N]^{-1}}.
\end{align}
Hence, by choosing $\delta$ sufficiently small, we can ensure that $\mathcal C_2$ contains the unique solution to (18) of Theorem \ref{thm:logG} of the main text. Further, a direct computation shows that
\[
\frac{d^2}{d\zeta^2} \widehat{\Psi}(i\zeta)\bigg|_{\zeta = \zeta_*} = -\left[\frac{1}{\frac{P}{2N}+\zeta_*}+\sum_{\ell=1}^L \frac{1}{\frac{N_\ell}{2N}+\zeta_*}\right] < 0,
\]
proving that that $i\zeta_*$ is a non-degenerate critical point of $\widehat{\Psi}$.
Defining $\zeta_0=\zeta_* + \zeta_{**}/N$ and
\begin{align*}
    \Psi_0(\zeta_0) &= \widehat{\Psi}(i\zeta_0) + \frac{1}{N}\twiddle{\Psi}(i\zeta_0),
\end{align*}
the Laplace method gives
\begin{align*}
    \log \mathrm{Den}_{\sum_{\ell=0}^L\log \phi_\ell}(0) &= \log\lr{\frac{N}{2\pi}\int_{\mathcal C_1(T,C_\delta)\cup \mathcal C_2(C_\delta)\cup \mathcal C_3(T,C_\delta)}\exp\left[N\Psi(z)\right]dz}\\
    &= N\Psi_0(\zeta_0) + \frac{1}{2}\log\lr{\frac{N}{2\pi\Psi_0''(\zeta_0)}}\\
    &= -N\zeta_0\left[\log\lr{\frac{\norm{\theta_*}^2}{\sigma^{2(L+1)}\alpha_0}} + L + 1 - \log\lr{1+2\zeta_0\frac{N}{P}} - \sum_{\ell=1}^L \log\lr{1+2\zeta_0\frac{N}{N_\ell}}\right]\\
    &\qquad + \frac{N}{2}\left[\frac{P}{N} \log\lr{1+2\zeta_0\frac{N}{P}} + \frac{N_\ell}{N}\sum_{\ell=1}^L \log\lr{1+2\zeta_0\frac{N}{N_\ell}}\right]\\
    &\qquad + \frac{1}{2}\log\lr{1+2\zeta_0\frac{N}{P}} + \lr{k+\frac{1}{2}}\sum_{\ell=1}^L \log\lr{1+2\zeta_0\frac{N}{N_\ell}}\\
    &\qquad + \frac{1}{2}\log(N) - \frac{1}{2}\log(2\pi) - \frac{1}{2}\log\left[\frac{2}{2\zeta_0+P/N} + \sum_{\ell=1}^L\frac{2}{2\zeta_0+N/N_\ell}\right].
\end{align*}
Taking
\[
N_\ell = N, \quad P = \alpha N
\]
and observing that
\[
\log\lr{\frac{\norm{\theta_*}^2}{\sigma^{2(L+1)}\alpha_0}} = L \log(1+2\zeta_*) + \log\lr{1+\frac{2\zeta_*}{\alpha}}
\]
simplifies the density to
\begin{align*}
    \log \mathrm{Den}_{\sum_{\ell=0}^L\log \phi_\ell}(0) &= \frac{N}{2}\left[\alpha\lr{ \log\lr{1+\frac{2\zeta_*}{\alpha}} - \frac{2\zeta_*}{\alpha}}+ L\lr{\log\lr{1+2\zeta_*} - 2\zeta_*}\right] + \frac{1}{2}\log(N) - \frac{1}{2}\log(\pi)\\
    &\qquad + \frac{1}{2}\log\lr{1+\frac{2\zeta_*}{\alpha}} + \lr{k+\frac{1}{2}}L \log\lr{1+2\zeta_*} - \frac{1}{2}\log\lr{\frac{1}{\alpha+2\zeta_*} + \frac{L}{1+2\zeta_*}}.
\end{align*}
Solving for the $G$-function given \eqref{eq:G-gamma}, we conclude that
\begin{align*}
    \log G &= \log \mathrm{Den}_{\sum_{\ell=0}^L\log \phi_\ell}(0) + L \log\left[\Gamma\left(\frac{N}{2}+k+1\right)\right] + \log\left[\Gamma\left(\frac{N\alpha}{2}+1\right)\right] \\
    &\qquad - \log \frac{\norm{\theta_*}^2}{\alpha_0} - L\log\left(\frac{N}{2\sigma^2}\right) -\log\left(\frac{N\alpha}{2\sigma^2}\right)\\
    &= \frac{N}{2}\left\{\alpha\left[\log\lr{\frac{N\alpha}{2}} + \log\lr{1+\frac{2\zeta_*}{\alpha}} - \lr{1+\frac{2\zeta_*}{\alpha}}\right]+ L\left[\log\lr{\frac{N}{2}} + \log\lr{1+2\zeta_*} - (1+2\zeta_*)\right]\right\}\\
    &\qquad + \frac{L(2k-1)}{2}\left[\log\lr{\frac{N}{2}} + \log(1+2\zeta_*)\right] + \frac{L}{2}\log(2\pi) - \frac{1}{2}\log\lr{1 + \frac{\alpha+2\zeta_*}{1+2\zeta_*}L}.
\end{align*}
The leading order part of this expression reproduces (16) of Theorem \ref{thm:logG} of the main text, and taking differences between the value at a fixed $k$ and at $k=0$ gives (19). Note that at $L=0$, direct computation gives
\begin{align*}
    \log G &= -\frac{\norm{\theta_*}^2}{\alpha_0}\frac{N\alpha}{2\sigma^2} + \frac{N\alpha}{2}\left(\log\frac{N\alpha}{2\sigma^2} + \log\frac{\norm{\theta_*}^2}{\alpha_0}\right),
\end{align*}
which we see is reproduced by the above. The derivations of the formulas in cases (b) and (c) of Theorem \ref{thm:logG} are very similar to those of (a). So we indicate only the salient differences, starting with case (b). We will actually consider the following somewhat more general regime:
\[
N:=\min \set{N_1,\ldots, N_L}\gives \infty,\quad  P,N_0,N,L \gives \infty,\quad \frac{P}{N_0}\gives \alpha_0, \quad \sum_{\ell=1}^L \frac{1}{N_\ell} \gives \lpre,
\]
with $\alpha_0\in (0,1)$ and $\lpre\in (0,\infty)$. Our starting point is again Lemma \ref{lem:contour}. The first modification in the proof is that we must redefine the contours $\mathcal C_j,\, j=0,\ldots, 4$ by replacing 
\[
T\mapsto T/L.
\]
Next, Lemma \ref{lem:Psi-est} now reads.
\begin{lemma}\label{lem:Psi-estb}
There exist $c,T_0>0$ such that
\begin{equation}\label{eq:Psi-est-1b}
    \sup_{\substack{\abs{z}>T/L\\ z\in \R}} \frac{\Re \Psi(z)}{1+\abs{z}} \leq - c,\qquad \forall T\geq T_0.
\end{equation}
Moreover, for any $T,\delta>0$, writing
\begin{equation}\label{eq:S-defb}
    S_{N,L,\delta,T}:=\set{z\in \C~|~ \abs{z}< T/L, \Im(z) >C_\delta},
\end{equation}
there exists $C>0$ such that
\begin{equation}\label{eq:Psi-est-2b}
    \sup_{z\in S_{N,L,\delta, T}}\max\set{\abs{\Psi(z)},\abs{ \frac{d}{dz}\Psi(z)}} \leq C.
\end{equation}
\end{lemma}
The proof of Lemma \ref{lem:Psi-estb} is essentially identical to that of Lemma \ref{lem:Psi-est} with the main difference being that the analogous estimates in \eqref{eq:RePsi-bound} must now be summed from $\ell=0$ to $\ell=L$, which involves a growing number of terms. In each term, for $T$ sufficiently large, the real part of $\Psi_\ell(z)$ is bounded above by $-c*(1+\abs{z})$ as soon as $\abs{z}/L > T$. As a result, the previous Lemma shows that, for any $K,\delta >0$ there exists $T>0$ such that
\begin{equation}
\frac{N}{2\pi}\int_{\mathcal C} \exp\left[N\Psi(z)\right] dz = \frac{N}{2\pi}\int_{\mathcal C_1(T/L,C_\delta)\cup \mathcal C_2(C_\delta) \cup \mathcal C_3(T/L,C_\delta)} \exp\left[N\Psi(z)\right] dz  + O(e^{-KN}).
\end{equation}
Next, exactly as in the derivation \eqref{eq:Psi-crit-2} and \eqref{eq:zeta-star2-def}, we find that $\Psi$ has a unique critical point $\zeta_*+\frac{1}{N}\zeta_{**}+O(N^{-2})$ along $\mathcal C_2$ and no critical points along $\mathcal C_1(T/L,C_\delta)$ and $\mathcal C_2(T/L,C_\delta)$. Moreover, recalling that $\sigma^2=1$ in this regime, a direct inspection of \eqref{eq:Psi-crit-2} and \eqref{eq:zeta-star2-def} reveals that there is a constant $C_*>0$ so that
\[
\abs{\zeta_*+\frac{1}{N}\zeta_{**}}\leq C_*\lr{\frac{1}{L}+\frac{1}{N}}.
\]
This allows us to choose $\delta$ in the definition of $ C_\delta$ to be independent of $L,N, P,N_0$. The remainder of the derivation is a direct computation of the value, first derivative, and second derivative of $\Psi$ at the its critical point followed by a straight-forward application of the Laplace method. Namely, the critical point of $\Psi$ on the contour $\mathcal C_2$ takes the form
\begin{align*}
\frac{d}{dz}\Psi(i\zeta) =0\quad \Longleftrightarrow \quad \zeta = \zeta_* +\frac{1}{N} \zeta_{**}+O(N^{-2}),
\end{align*}
where the critical point is given in terms of $\lpre = \sum_{\ell=1}^N N_\ell^{-1}$:
\begin{align*}
    \zeta_*&=0\\
     \zeta_{**} &= \frac{1}{2}\left[\frac{1}{\lpre}\log\lr{\frac{\norm{\theta_*}^2}{\alpha_0}} - (2k+1)\right].
\end{align*}
The corresponding critical value is given by
\begin{align*}
\Psi\lr{i\lr{\zeta_*+\frac{1}{N}\zeta_{**}}} &= -\frac{\lpre}{4N}\lr{2k+1 - \frac{1}{\lpre}\log \frac{\norm{\theta_*}^2}{\alpha_0}}^2,
\end{align*}
and the Hessian is, to leading order,
\begin{align*}
\frac{d^2}{d\zeta^2}\Psi\lr{i\lr{\zeta_*+\frac{1}{N}\zeta_{**}}} & = - 2N\lpre < 0,
\end{align*}
showing that $i\zeta_{**}/N$ is a non-degenerate critical point.
Including the terms from the Gamma function prefactors, we obtain the $G$ function
\begin{align*}
    \log \G{L+1,0}{0,L+1}{\frac{\norm{\theta_*}^2}{4M}}{-}{\frac{P}{2},\frac{\mathbf{N}}{2}+k} &= \sum_{\ell=1}^L \frac{N_\ell}{2}\left[\log\lr{\frac{N_\ell}{2}}-1\right] + \frac{P}{2}\left[\log\lr{\frac{P}{2}}-1\right] + \frac{L}{2}\log(2\pi)\\
    &\quad + \lr{k-\frac{1}{2}}\sum_{\ell=1}^L \log\lr{\frac{N_\ell}{2}} - \frac{1}{2}\log\lr{\frac{P}{2}} + \lr{k-\frac{1}{2}}\log\lr{\frac{\norm{\theta_*}^2}{\alpha_0}}\\
    &\quad  - \frac{1}{12}\lpre - \frac{1}{4\lpre}\left[\log\lr{\frac{\norm{\theta_*}^2}{\alpha_0}}\right]^2 - \frac{1}{2}\log(2\lpre).
\end{align*}
When specialized to the case when $N_1=\cdots=N_L=N$, these are the results stated in (19) and (20) of Theorem \ref{thm:logG}. Finally, for case (c), we our results apply to the  regime where
\[
N:=\min \set{N_1,\ldots, N_L}\gives \infty,\quad  P,N_0,N \gives \infty,\quad \frac{P}{N_0}\gives \alpha_0, \quad P\sum_{\ell=1}^L \frac{1}{N_\ell} \gives \lpost,
\]
with $\alpha_0\in (0,1)$ and $\lpost\in (0,\infty)$. The analysis in this case mirrors almost exactly case (b), but we use the variable substitution $t\mapsto Pt$ when changing from $\Phi$ to $\Psi$, and use the fact that both $P/N$ and $L/N$ vanish. Here, we record only the result:
\begin{align*}
    \log \G{L+1,0}{0,L+1}{\frac{\norm{\theta_*}^2}{4M}}{-}{\frac{P}{2},\frac{\mathbf{N}}{2}+k} &= \frac{P}{2}\left[\log\lr{\frac{P}{2}}-1 + \log(1+t_*) -t_*\lr{1 + \frac{\lpost t_*}{2}}\right] - \log P \\
    &\quad + \frac{L}{2}\log(2\pi) + \sum_{\ell=1}^L \frac{N_\ell}{2}\left[\log\lr{\frac{N_\ell}{2}}-1\right] + \frac{1}{2}(2k-1)\log\lr{\frac{N_\ell}{2}} \\
    &\quad - \frac{1}{2}\log\lr{\lpost + \frac{1}{1+t_*}} + \frac{1}{2}\left[(2k+1)\lpost t_* + \log(1+t_*)\right],
\end{align*}
where $t_*$ is the unique solution to 
\[
e^{\lpost t_*}\lr{1+t_*}=\frac{\norm{\theta_*}^2}{\alpha_0}.
\]
When specialized to the case when $N_1=\cdots=N_L=N$, this gives the results stated in (21) and (22) of Theorem \ref{thm:logG}. This completes the proof of Theorem \ref{thm:logG}. \hfill $\square$

\subsection{Proof of Theorem 
 \ref{thm:bayesfeature}
}\label{sec:bayesfeature-pf}
Consider the setup from Theorem \ref{thm:Z-form}. We adopt the notation
\begin{align*}
    z &= \frac{\norm{\theta_*}^2}{4M} = \frac{\nu}{\sigma^{2(L+1)}}\frac{P}{2}\lr{\frac{N}{2}}^L
\end{align*}
so that we have variance
\begin{align*}
    \Var_\mathrm{post}[f(x)] &= \frac{\norm{x_\perp}^2}{2}\frac{\norm{\theta_*}^2}{z} \frac{\G{L+1,0}{0,L+1}{z}{-}{\frac{P}{2},\frac{N_1}{2}+1,\dots,\frac{N_L}{2}+1}}{\G{L+1,0}{0,L+1}{z}{-}{\frac{P}{2},\frac{N_1}{2},\dots,\frac{N_L}{2}}}\\
    &= \frac{\norm{x_\perp}^2}{2}\norm{\theta_*}^2 \frac{\G{L+1,0}{0,L+1}{z}{-}{\frac{P}{2}-1,\frac{N_1}{2},\dots,\frac{N_L}{2}}}{\G{L+1,0}{0,L+1}{z}{-}{\frac{P}{2},\frac{N_1}{2},\dots,\frac{N_L}{2}}}.
\end{align*}
A Bayes-optimal $\sigma^2$ implies
\begin{align*}
    \frac{\partial Z_\infty(0)}{\partial \sigma^2} &= 0 \implies \frac{d}{d z} \G{L+1,0}{0,L+1}{z}{-}{\frac{P}{2},\frac{N_1}{2},\dots,\frac{N_L}{2}} = 0.
\end{align*}
Using the identity
\begin{align*}
    \frac{d}{dz}\left[z^{-b_1}\G{m,n}{p,q}{z}{\mathbf{a_p}}{\mathbf{b_q}}\right] = -z^{-1-b_1}\G{m,n}{p,q}{z}{\mathbf{a_p}}{b_1+1,b_2,\dots,b_q}
\end{align*}
which holds when $m \geq 1$, we find that
\begin{align*}
    \frac{d}{dz} \G{m,n}{p,q}{z}{\mathbf{a_p}}{\mathbf{b_q}} &= \frac{1}{z}\left[b_1\G{m,n}{p,q}{z}{\mathbf{a_p}}{\mathbf{b_q}} - \G{m,n}{p,q}{z}{\mathbf{a_p}}{b_1+1, b_2, \dots, b_q} \right]
\end{align*}
and thus
\begin{align}\label{eq:sigmazero}
    \frac{P}{2}\G{L+1,0}{0,L+1}{z}{-}{\frac{P}{2},\frac{N_1}{2},\dots,\frac{N_L}{2}} &= \G{L+1,0}{0,L+1}{z}{-}{\frac{P}{2}+1,\frac{N_1}{2},\dots,\frac{N_L}{2}}.
\end{align}
To show both directions of Theorem \ref{thm:bayesfeature}, it hence suffices to show that to leading order in the large-$P$ limit,
\begin{align}
\label{eq:gratio}
    \frac{\G{L+1,0}{0,L+1}{z}{-}{\frac{P}{2}-1,\frac{N_1}{2},\dots,\frac{N_L}{2}}}{\G{L+1,0}{0,L+1}{z}{-}{\frac{P}{2},\frac{N_1}{2},\dots,\frac{N_L}{2}}} = \frac{\G{L+1,0}{0,L+1}{z}{-}{\frac{P}{2},\frac{N_1}{2},\dots,\frac{N_L}{2}}}{\G{L+1,0}{0,L+1}{z}{-}{\frac{P}{2}+1,\frac{N_1}{2},\dots,\frac{N_L}{2}}} = \frac{2}{P},
\end{align}
which ensures variance
\begin{align*}
    \Var_\mathrm{post}[f(x)] &= \lr{\frac{\norm{x_\perp}^2}{2}\norm{\theta_*}^2}\lr{\frac{2}{P}}= \frac{\norm{x_\perp}^2}{N_0} \frac{\norm{\theta_*}^2}{\alpha_0}.
\end{align*}
This simultaneously sets the predictor variance to $\nu \Sigma_\perp$ that we recover above. Conversely, if the variance is given by $\nu \Sigma_\perp$ to leading order, \eqref{eq:gratio} implies that $\partial Z_\infty(0)/\partial \sigma^2 = 0$.

To show \eqref{eq:gratio}, we evaluate a saddle point approximation of the relevant $G$ function
\begin{align*}
    \G{L+1,0}{0,L+1}{z}{-}{\frac{P}{2}+k,\frac{N_1}{2},\dots,\frac{N_L}{2}}
\end{align*}
with constant $L, N_1, \dots, N_L$. Similarly to the result of Lemma~\ref{lem:contour}, we use the density of independent $\Gamma$ random variables
\begin{align*}
\phi_j \sim \begin{cases}\Gamma\left(\frac{N_j}{2}+1, \frac{2\sigma^2}{N_j}\right), &j=1,\dots,L\\ \Gamma\left(\frac{P}{2}+k+1, \frac{2\sigma^2}{P}\frac{\alpha_0}{\norm{\theta_*}^2}\right), &j=0\end{cases}.
\end{align*}
Note that unlike in Lemma~\ref{lem:contour}, these variables offset $P$ by $k$, not $N_\ell$. The density is thus
\begin{align*}
\mathrm{Den}_{\phi_0\cdots \phi_{L}}(1) &= \frac{\norm{\theta_*}^2}{4M}\G{L+1,0}{0,L+1}{\frac{\norm{\theta_*}^2}{4M}}{-}{\frac{P}{2}+k,\frac{{\bf N}}{2}}\left[\Gamma\left(\frac{P}{2}+k+1\right)\right]^{-1}\prod_{\ell=1}^L\left[\Gamma\left(\frac{N_\ell}{2}+1\right)\right]^{-1}\\
&= \frac{1}{2\pi}\int_{\mathcal C} \exp\left[\Phi(z)\right]dz,
\end{align*}
where $\mathcal C\subseteq\C$ is the contour that runs along the real line from $-\infty$ to $\infty$ and
\begin{align*}
\notag   \Phi(z)&= -iz\log\lr{\frac{2\sigma^2}{P}\frac{\alpha_0}{\norm{\theta_*}^2}}+\log\lr{\frac{\Gamma\lr{\frac{P}{2}+k+1-iz}}{\Gamma\lr{\frac{P}{2}+k+1}}}\\
  &+\sum_{\ell=1}^L\left\{-iz\log\lr{\frac{2\sigma^2}{N_\ell}}+\log\lr{\frac{\Gamma\lr{\frac{N_\ell}{2}+1-iz}}{\Gamma\lr{\frac{N_\ell}{2}+1}}}    \right\}.
\end{align*}
The fixed point equation $d\Phi(i\zeta)/d\zeta=0$ is solved by $\zeta_*$ satisfying
\begin{align*}
    \sum_{\ell=1}^L - \log\lr{\frac{N}{2}} + \psi\lr{\frac{N_\ell}{2}+\zeta_*+1} = \log\lr{\frac{\|\theta_*\|^2}{\alpha_0 \sigma^{2(L+1)}}},
\end{align*}
i.e., by $\zeta_* = O(1)$. Directly applying Laplace's method
\begin{align*}
    \log \frac{1}{2\pi}\int_{\mathcal C} \exp\left[\Phi(z)\right]dz &= \Phi(z_*) - \frac{1}{2}\log(2\pi) - \frac{1}{2}\log\Phi''(z_*)
\end{align*}
and evaluating the ratios in \eqref{eq:gratio}, we find
\begin{align*}
    \log \frac{\G{L+1,0}{0,L+1}{z}{-}{\frac{P}{2}-1,\frac{N_1}{2},\dots,\frac{N_L}{2}}}{\G{L+1,0}{0,L+1}{z}{-}{\frac{P}{2},\frac{N_1}{2},\dots,\frac{N_L}{2}}} &= \log \frac{\G{L+1,0}{0,L+1}{z}{-}{\frac{P}{2},\frac{N_1}{2},\dots,\frac{N_L}{2}}}{\G{L+1,0}{0,L+1}{z}{-}{\frac{P}{2}+1,\frac{N_1}{2},\dots,\frac{N_L}{2}}}= \log\lr{\frac{2}{P}} + O\lr{\frac{1}{P}},
\end{align*}
which we note is independent of $\sigma^2$ to leading order. That is, the large dataset overwhelms the prior, causing all choices of $\sigma^2$ to become optimal. This completes the proof.

\subsection{Model Evidence}
\label{sec:evidence}
In the infinite-depth case $L=\lpre N$ as $N,P\to\infty$, we find that not only does setting $\sigma^2=1,\lpre>0$ give desirable posteriors, but it also enjoys a significant preference with respect to model evidence.

\begin{corollary}[Bayesian Preference for Infinite Depth]
\label{cor:infevd}
As in the setting of Theorem 6, fix $\lpre>0$. For each fixed $L\geq 0$ 
\[
\lim_{N\gives \infty}e^{cN}\frac{Z_\infty\lr{0~|~L,N_\ell=N,\sigma^2=1,X_{N_0},Y_{N_0}}}{Z_\infty\lr{0~|~L=N\lpre,N_\ell=N,\sigma^2=1,X_{N_0},Y_{N_0}}} =1,
\]
where
\[
c = -\frac{\alpha}{2}\left[\log\lr{1+\frac{z_*}{\alpha}} - \frac{z_*}{\alpha}\right] - \frac{L}{2}\left[\log(1+z_*)-z_*\right]
\]
satisfies $c \leq 0$, and $z_*$ is defined as in \eqref{eq:Psi-crit} of Theorem \ref{thm:logG}, i.e.,
\begin{equation}
\frac{\nu}{\sigma^{2(L+1)}} = \lr{1+\frac{z_*}{\alpha}}\lr{1+z_*}^L,\quad z_* > \max\set{-1,-\alpha}.
\end{equation}
\end{corollary}

This shows that, when $\sigma^2=1$, any choice of $\lpre$ results in exponentially greater evidence than a network with finitely many hidden layers in the large-$N$ limit. We omit the proof, since it is a direct manipulation of the reported evidence in the main statement of Theorems 4 and 6.

Moreover, the Bayes-optimal depth is given by an $O(1)$ choice of $\lpre$. That is, unlike in the finite-depth limit where $L_*\to\infty$ when using an ML prior, choosing $L=O(N)$ has an attainable optimal depth that maximizes Bayesian evidence.
\begin{corollary}[Bayes-optimal infinite depth]
\label{cor:optevd}
In the setting of Theorem 6, we have
\[
\lambda_{\mathrm{prior},*} = \sqrt{1 + \log(\nu)^2}-1 = \argmax_{\lambda}\lim_{\substack{P,N_\ell \gives \infty\\ P/N_0\gives \alpha_0 \in (0,1)\\ \lpre(N_1,\ldots, N_L)\gives \lambda}} Z_\infty\lr{{\bf 0}~|~L,N_\ell,\sigma^2=1,X_{N_0},Y_{N_0}}.
\]
Moreover, for any $\lambda>0$, there exists $c>0$ such that we have
\begin{align*}
    c < \lim_{\substack{P,N_\ell \gives \infty\\ P/N_0\gives \alpha_0 \in (0,1)}} \frac{Z_\infty\lr{{\bf 0}~|~L=N\lambda,N_\ell=N,\sigma^2=1,X_{N_0},Y_{N_0}}}{ Z_\infty\lr{{\bf 0}~|~L=N\lambda_{\text{prior},*},N_\ell=N,\sigma^2=1,X_{N_0},Y_{N_0}}}\leq 1,
\end{align*}
i.e., the ratio of evidences is lower-bounded by a constant.
\end{corollary}

\begin{proof}
In Corollary~\ref{cor:optevd}, we compute the Bayes-optimal neural network depth by maximizing Bayesian evidence~\cite{mackay1992bayesian}. A similar computation to the following is used to find Bayes-optimal parameters throughout the main text; we provide the proof to Corollary~\ref{cor:optevd} as an example. To evaluate $\partial \log Z / \partial \lpre = 0$ for the partition function given by Theorem \ref{thm:Z-form}, we take the partition function with $N_1 = \dots = N_L = N$ obtained from case (b) of Theorem \ref{thm:logG},
\begin{align*}
    \log Z_\infty(0) &= \frac{P}{2}\log\lr{\frac{4\pi}{\|\theta_*\|^2}} + \frac{P}{2}\left[\log\lr{\frac{P}{2}}-1\right] -\frac{1}{2}\log\lr{\frac{P}{2}}\\
    &\quad  - \frac{1}{2}\log(2\lpre) - \frac{1}{4\lpre}\lr{\lpre + \log\lr{\frac{\|\theta_*\|^2}{\alpha_0}}}^2,
\end{align*}
and we evaluate the derivative with variable substitution $\nu = \|\theta_*\|^2/\alpha_0$, yielding
\begin{align*}
    \frac{\partial \log Z_\infty(0)}{\partial \lpre} &= \frac{\log(\nu)^2-\lpre(2+\lpre)}{4\lpre^2} = 0.
\end{align*}
Solving gives
\begin{align*}
    \lambda_{\mathrm{prior},*} &= \sqrt{1 + \log^2 \nu} - 1 \geq 0.
\end{align*}
Moreover, the second derivative is
\begin{align*}
    \frac{\partial^2 \log Z_\infty(0)}{\partial \lpre^2} &= \frac{\lpre-\log^2\nu}{2\lpre^3},
\end{align*}
which at $\lambda_{\mathrm{prior},*}$ is
\begin{align*}
    \frac{\partial^2 \log Z_\infty(0)}{\partial \lpre^2} &= -\frac{\lambda_{\mathrm{prior},*}+1}{2\lambda_{\mathrm{prior},*}^2} < 0.
\end{align*}
Hence, Bayesian evidence is maximized at $\lambda_{\mathrm{prior},*} = O(1)$.

Moreover, evaluating the evidence at a different choice of fixed $\lpre$ such that $\lpre \neq \lambda_{\mathrm{prior},*}$, we see that the evidence ratio is
\begin{align*}
    \log Z_\infty(0; \lpre) - \log Z_\infty(0; \lambda_{\mathrm{prior},*}) &= -\frac{1}{2}\log\lr{\frac{\lpre}{\lambda_{\mathrm{prior},*}}} - \frac{1}{4\lpre}\lr{\lpre + \log\lr{\frac{\|\theta_*\|^2}{\alpha_0}}}^2 \\
    &\quad + \frac{1}{4\lambda_{\mathrm{prior},*}}\lr{\lambda_{\mathrm{prior},*} + \log\nu}^2,
\end{align*}
which is a constant.
\end{proof}

\subsection{Proof of Theorem \ref{thm:scaling}
}\label{sec:scaling-pf}
In order to prove Theorem \ref{thm:scaling}, we expand $\Delta(\log G)[k]$ in the case of $L=\lpre N, \, P/N=\alpha, \, \sigma^2=1$ to higher order than reported in Theorem \ref{thm:logG}. Specifically, we find
\begin{align*}
    \Delta(\log G, k=1) &:= \log \G{L+1,0}{0,L+1}{\frac{\norm{\theta_*}^2}{4M}}{-}{\frac{P}{2},\frac{\mathbf{N}}{2}+1} - \log \G{L+1,0}{0,L+1}{\frac{\norm{\theta_*}^2}{4M}}{-}{\frac{P}{2},\frac{\mathbf{N}}{2}}\\
    &= \lpre N \log\lr{\frac{N}{2}} + \log(\nu) + \frac{c}{N}
\end{align*}
for
\begin{align*}
    c &= -\frac{8\lpre}{3}+2\lr{1+\log(\nu)} + \lr{\frac{1}{\alpha}+\log(\nu)}\lr{1-\frac{\log(\nu)}{\lpre}}.
\end{align*}
Applying Theorem \ref{thm:Z-form}, the difference in variance compared to infinite $N$ is
\begin{align*}
    \Var_{\mathrm{post}}\left[f(x)\right] - \lim_{N\to\infty} \Var_{\mathrm{post}}\left[f(x)\right] &= \frac{c}{N}\nu \Sigma_\perp \propto \frac{1}{N} \propto \frac{1}{P} \propto \frac{1}{L}.
\end{align*}

% \subsection{Proof of Theorems \ref{thm:finite-L}-\ref{thm:LN}}
% \label{sec:highorder}
% State theorems in general form (with $N_\ell$), find optimal infinite depth $\lambda_\mathrm{pre,*}$, etc.

\subsection{Proof of Theorem %\ref{thm:dd}
\ref{thm:dd}
}\label{sec:dd-pf}
To prove Theorem \ref{thm:dd} we begin with the bias-variance decomposition:
\[
\left\langle f(x)-V_0x \right \rangle^2 = \lr{\theta_*x_{||}-V_0x}^2 + \frac{\norm{x_\perp}^2\norm{\theta_*}^2}{P},
\]
where $x_\perp = x - x_{||}$ and
\[
x_{||} = \im(X) \im(X)^T x.
\]
Consider first the case when $\alpha_0 < 1$. Then
\[
\theta_* = V_0 + \epsilon (X^TX)^{-1}X^T.
\]
Thus, the error introduced by the bias is
\begin{align*}
    \mathbb E\left[\lr{\theta_*x_{||}-V_0x}^2 \right] &= \mathbb E\left[(V_0 x_\perp)^2\right] + \mathbb E\left[(\epsilon (X^T X)^{-1} X^T x_{||})^2\right]\\
    &= \mathbb{E}\left[(V_0x)^2 - 2(V_0x)(V_0x_{||}) + (V_0x_{||})^2\right] + \sigma_\epsilon^2 \mathbb{E}\left[\|(X^T X)^{-1} X^T x_{||}\|^2\right]\\
    &= 1-\alpha_0 + \sigma_\epsilon^2 \mathbb{E}\left[\|(X^T X)^{-1} X^T x_{||}\|^2\right].
\end{align*}
We observe that, for $X^\dagger = (X^T X)^{-1} X^T$,
\begin{align*}
   \mathbb{E}\left[\|(X^T X)^{-1} X^T x_{||}\|^2\right] &= \E{\tr\lr{x_{||}^T \lr{X^{\dagger}}^T X^{\dagger}x_{||}}}\\
   &= \E{\tr\lr{ \lr{X^{\dagger}}^T X^{\dagger}x_{||}x_{||}^T}}\\
   &= \frac{\alpha_0}{1-\alpha_0}.
\end{align*}
Hence, the bias is
\begin{align}
\label{eq:biasl1}
    \mathbb E\left[\lr{\theta_*x_{||}-V_0x}^2 \right] &= 1-\alpha_0 + \frac{\alpha_0}{1-\alpha_0}\sigma_\epsilon^2.
\end{align}
The error introduced by the variance is
\begin{align*}
    \mathbb{E}\left[\frac{\norm{x_\perp}^2\norm{\theta_*}^2}{P}\right] &= \frac{1}{P}\mathbb{E}\left[\left(\|V_0\|^2 + \|\epsilon(X^TX)^{-1}X^T\|^2\right)\left(\|x\|^2 + \|x_{||}\|^2 - 2 \|\im(X)^Tx\|^2\right)\right]\\
    &= \frac{1-\alpha_0}{P}\mathbb{E}\left[\|x\|^2\left(1 + \sigma_\epsilon^2\|(X^TX)^{-1}X^T\|^2\right)\right]\\
    &= \left(\frac{1}{\alpha_0} - 1\right)\left(1 + \sigma_\epsilon^2\mathbb{E}\left[\|(X^TX)^{-1}X^T\|^2\right]\right).
\end{align*}
Observing that
\begin{align*}
\mathbb{E}\left[\left\|(X^TX)^{-1}X^T\right\|^2\right] &= \mathbb{E}\left[\tr\lr{(X^TX)^{-1}X^TX(X^TX)^{-1}}\right] = \mathbb{E}\left[\tr\lr{(X^TX)^{-1}}\right] = \frac{\alpha_0}{1-\alpha_0}
\end{align*}
from the $-1$st moment of the Marchenko-Pastur distribution for $\alpha_0 < 1$, we obtain
\begin{align*}
    \left\langle f(x)-V_0x \right \rangle^2 &= 1-\alpha_0 + \frac{\alpha_0}{1-\alpha_0}\sigma_\epsilon^2 + \left(\frac{1}{\alpha_0} - 1\right)\left(1 + \frac{\alpha_0}{1-\alpha_0}\sigma_\epsilon^2\right)\\
    &= \frac{1}{\alpha_0} - \alpha_0 + \frac{\sigma_\epsilon^2}{1-\alpha_0}.
\end{align*}
In the case of $\alpha_0 > 1$, the variance is zero since $\|x_\perp\|^2 = 0$. Given
\begin{align*}
    \theta_* &= V_0 + \epsilon X^T(XX^T)^{-1},
\end{align*}
the total error originates from the bias, i.e.,
\begin{align}
    \left\langle f(x)-V_0x \right \rangle^2 &= \mathbb{E}\left[(\epsilon (X^TX)^{-1}X^T x)^2\right] \nonumber\\
    &= \frac{\sigma_\epsilon^2}{\alpha_0 - 1} \label{eq:biasg1}
\end{align}
similarly to the bias computation for $\alpha_0 < 1$.

\bibliography{new_bib}

\newcommand{\etalchar}[1]{$^{#1}$}
\begin{thebibliography}{HNPSD22}

\bibitem[ACGH19]{arora2018convergence}
Sanjeev Arora, Nadav Cohen, Noah Golowich, and Wei Hu.
\newblock A convergence analysis of gradient descent for deep linear neural
  networks.
\newblock {\em ICLR}, 2019.

\bibitem[ACH18]{arora2018optimization}
Sanjeev Arora, Nadav Cohen, and Elad Hazan.
\newblock On the optimization of deep networks: Implicit acceleration by
  overparameterization.
\newblock {\em ICML}, 2018.

\bibitem[ALP22]{adlam2019random}
Ben Adlam, Jake Levinson, and Jeffrey Pennington.
\newblock A random matrix perspective on mixtures of nonlinearities for deep
  learning.
\newblock {\em AISTATS}, 2022.

\bibitem[AP20]{adlam2020neural}
Ben Adlam and Jeffrey Pennington.
\newblock The neural tangent kernel in high dimensions: Triple descent and a
  multi-scale theory of generalization.
\newblock In {\em International Conference on Machine Learning}, pages 74--84.
  PMLR, 2020.

\bibitem[APP{\etalchar{+}}22]{ariosto2022statistical}
S~Ariosto, R~Pacelli, M~Pastore, F~Ginelli, M~Gherardi, and P~Rotondo.
\newblock Statistical mechanics of deep learning beyond the infinite-width
  limit.
\newblock {\em arXiv preprint arXiv:2209.04882}, 2022.

\bibitem[AS64]{abramowitz1964handbook}
Milton Abramowitz and Irene~A Stegun.
\newblock {\em Handbook of mathematical functions with formulas, graphs, and
  mathematical tables}, volume~55.
\newblock US Government printing office, 1964.

\bibitem[ASS20]{advani2020high}
Madhu~S Advani, Andrew~M Saxe, and Haim Sompolinsky.
\newblock High-dimensional dynamics of generalization error in neural networks.
\newblock {\em Neural Networks}, 132:428--446, 2020.

\bibitem[AZLL19]{allen2019learning}
Zeyuan Allen-Zhu, Yuanzhi Li, and Yingyu Liang.
\newblock Learning and generalization in overparameterized neural networks,
  going beyond two layers.
\newblock In {\em Advances in neural information processing systems}, pages
  6158--6169, 2019.

\bibitem[BDK{\etalchar{+}}21]{bahri2021explaining}
Yasaman Bahri, Ethan Dyer, Jared Kaplan, Jaehoon Lee, and Utkarsh Sharma.
\newblock Explaining neural scaling laws.
\newblock {\em arXiv preprint arXiv:2102.06701}, 2021.

\bibitem[BHMM19]{belkin2019reconciling}
Mikhail Belkin, Daniel Hsu, Siyuan Ma, and Soumik Mandal.
\newblock Reconciling modern machine-learning practice and the classical
  bias--variance trade-off.
\newblock {\em Proceedings of the National Academy of Sciences},
  116(32):15849--15854, 2019.

\bibitem[BHX20]{belkin2020two}
Mikhail Belkin, Daniel Hsu, and Ji~Xu.
\newblock Two models of double descent for weak features.
\newblock {\em SIAM Journal on Mathematics of Data Science}, 2(4):1167--1180,
  2020.

\bibitem[BLLT20]{bartlett2020benign}
Peter~L Bartlett, Philip~M Long, G{\'a}bor Lugosi, and Alexander Tsigler.
\newblock Benign overfitting in linear regression.
\newblock {\em Proceedings of the National Academy of Sciences}, 2020.

\bibitem[Bor14]{borel1914introduction}
Emile Borel.
\newblock {\em Introduction g{\'e}om{\'e}trique {\`a} quelques th{\'e}ories
  physiques}.
\newblock Gauthier-Villars, 1914.

\bibitem[CB18]{chizat2018global}
Lenaic Chizat and Francis Bach.
\newblock On the global convergence of gradient descent for over-parameterized
  models using optimal transport.
\newblock In {\em Advances in neural information processing systems}, pages
  3036--3046, 2018.

\bibitem[CKZ23]{cui2023optimal}
Hugo Cui, Florent Krzakala, and Lenka Zdeborov{\'a}.
\newblock Optimal learning of deep random networks of extensive-width.
\newblock {\em arXiv preprint arXiv:2302.00375}, 2023.

\bibitem[DF87]{diaconis1987dozen}
Persi Diaconis and David Freedman.
\newblock A dozen de finetti-style results in search of a theory.
\newblock In {\em Annales de l'IHP Probabilit{\'e}s et statistiques},
  volume~23, pages 397--423, 1987.

\bibitem[DL13]{damianou2013deep}
Andreas Damianou and Neil~D Lawrence.
\newblock Deep gaussian processes.
\newblock In {\em Artificial intelligence and statistics}, pages 207--215.
  PMLR, 2013.

\bibitem[DLT{\etalchar{+}}18]{du2017gradient}
Simon~S Du, Jason~D Lee, Yuandong Tian, Barnabas Poczos, and Aarti Singh.
\newblock Gradient descent learns one-hidden-layer cnn: Don't be afraid of
  spurious local minima.
\newblock {\em ICML}, 2018.

\bibitem[DZPS19]{du2018gradient}
Simon~S. Du, Xiyu Zhai, Barnabas Poczos, and Aarti Singh.
\newblock Gradient descent provably optimizes over-parameterized neural
  networks.
\newblock In {\em International Conference on Learning Representations}, 2019.

\bibitem[GJS{\etalchar{+}}20]{geiger2020scaling}
Mario Geiger, Arthur Jacot, Stefano Spigler, Franck Gabriel, Levent Sagun,
  St{\'e}phane d’Ascoli, Giulio Biroli, Cl{\'e}ment Hongler, and Matthieu
  Wyart.
\newblock Scaling description of generalization with number of parameters in
  deep learning.
\newblock {\em Journal of Statistical Mechanics: Theory and Experiment},
  2020(2):023401, 2020.

\bibitem[Han18]{hanin2018neural}
Boris Hanin.
\newblock Which neural net architectures give rise to exploding and vanishing
  gradients?
\newblock In {\em Advances in Neural Information Processing Systems}, 2018.

\bibitem[Han22]{hanin2022random}
Boris Hanin.
\newblock Random fully connected neural networks as perturbatively solvable
  hierarchies.
\newblock {\em arXiv preprint arXiv:2204.01058}, 2022.

\bibitem[HMRT22]{hastie2022surprises}
Trevor Hastie, Andrea Montanari, Saharon Rosset, and Ryan~J Tibshirani.
\newblock Surprises in high-dimensional ridgeless least squares interpolation.
\newblock {\em The Annals of Statistics}, 50(2):949--986, 2022.

\bibitem[HN20a]{hanin2019finite}
Boris Hanin and Mihai Nica.
\newblock Finite depth and width corrections to the neural tangent kernel.
\newblock {\em ICLR 2020}, 2020.

\bibitem[HN20b]{hanin2020products}
Boris Hanin and Mihai Nica.
\newblock Products of many large random matrices and gradients in deep neural
  networks.
\newblock {\em Communications in Mathematical Physics}, 376(1):287--322, 2020.

\bibitem[HNPSD22]{hron2022wide}
Jiri Hron, Roman Novak, Jeffrey Pennington, and Jascha Sohl-Dickstein.
\newblock Wide bayesian neural networks have a simple weight posterior: theory
  and accelerated sampling.
\newblock In {\em International Conference on Machine Learning}, pages
  8926--8945. PMLR, 2022.

\bibitem[HP21]{hanin2021non}
Boris Hanin and Grigoris Paouris.
\newblock Non-asymptotic results for singular values of gaussian matrix
  products.
\newblock {\em Geometric and Functional Analysis}, 31(2):268--324, 2021.

\bibitem[HR18]{hanin2018start}
Boris Hanin and David Rolnick.
\newblock How to start training: The effect of initialization and architecture.
\newblock In {\em Advances in Neural Information Processing Systems}, pages
  571--581, 2018.

\bibitem[JGH18]{jacot2018neural}
Arthur Jacot, Franck Gabriel, and Cl{\'e}ment Hongler.
\newblock Neural tangent kernel: Convergence and generalization in neural
  networks.
\newblock In {\em Advances in neural information processing systems}, pages
  8571--8580, 2018.

\bibitem[Kaw16]{kawaguchi2016deep}
Kenji Kawaguchi.
\newblock Deep learning without poor local minima.
\newblock In {\em Advances in Neural Information Processing Systems}, pages
  586--594, 2016.

\bibitem[KMH{\etalchar{+}}20]{kaplan2020scaling}
Jared Kaplan, Sam McCandlish, Tom Henighan, Tom~B Brown, Benjamin Chess, Rewon
  Child, Scott Gray, Alec Radford, Jeffrey Wu, and Dario Amodei.
\newblock Scaling laws for neural language models.
\newblock {\em arXiv preprint arXiv:2001.08361}, 2020.

\bibitem[LBN{\etalchar{+}}18]{lee2017deep}
Jaehoon Lee, Yasaman Bahri, Roman Novak, Samuel~S Schoenholz, Jeffrey
  Pennington, and Jascha Sohl-Dickstein.
\newblock Deep neural networks as gaussian processes.
\newblock {\em ICML 2018 andarXiv:1711.00165}, 2018.

\bibitem[LNR22]{li2022neural}
Mufan~Bill Li, Mihai Nica, and Daniel~M Roy.
\newblock The neural covariance sde: Shaped infinite depth-and-width networks
  at initialization.
\newblock {\em NeurIPS 2022}, 2022.

\bibitem[LS21]{li2021statistical}
Qianyi Li and Haim Sompolinsky.
\newblock Statistical mechanics of deep linear neural networks: The
  backpropagating kernel renormalization.
\newblock {\em Physical Review X}, 11(3):031059, 2021.

\bibitem[LZB22]{liu2022loss}
Chaoyue Liu, Libin Zhu, and Mikhail Belkin.
\newblock Loss landscapes and optimization in over-parameterized non-linear
  systems and neural networks.
\newblock {\em Applied and Computational Harmonic Analysis}, 59:85--116, 2022.

\bibitem[Mac92]{mackay1992bayesian}
David~JC MacKay.
\newblock Bayesian interpolation.
\newblock {\em Neural computation}, 4(3):415--447, 1992.

\bibitem[Mei36]{meijer1936whittakersche}
CS~Meijer.
\newblock {\"U}ber whittakersche bzw. besselsche funktionen und deren produkte.
\newblock {\em Nieuw Archief voor Wiskunde}, 18(2):10--29, 1936.

\bibitem[MM19]{mei2019generalizationb}
Song Mei and Andrea Montanari.
\newblock The generalization error of random features regression: Precise
  asymptotics and the double descent curve.
\newblock {\em Communications on Pure and Applied Mathematics}, 2019.

\bibitem[MMM21]{mei2021generalization}
Song Mei, Theodor Misiakiewicz, and Andrea Montanari.
\newblock Generalization error of random feature and kernel methods:
  hypercontractivity and kernel matrix concentration.
\newblock {\em Applied and Computational Harmonic Analysis}, 2021.

\bibitem[MMN18]{mei2018mean}
Song Mei, Andrea Montanari, and Phan-Minh Nguyen.
\newblock A mean field view of the landscape of two-layer neural networks.
\newblock {\em Proceedings of the National Academy of Sciences},
  115(33):E7665--E7671, 2018.

\bibitem[MZ22]{montanari2022interpolation}
Andrea Montanari and Yiqiao Zhong.
\newblock The interpolation phase transition in neural networks: Memorization
  and generalization under lazy training.
\newblock {\em The Annals of Statistics}, 50(5):2816--2847, 2022.

\bibitem[NBR{\etalchar{+}}21]{noci2021precise}
Lorenzo Noci, Gregor Bachmann, Kevin Roth, Sebastian Nowozin, and Thomas
  Hofmann.
\newblock Precise characterization of the prior predictive distribution of deep
  relu networks.
\newblock {\em Advances in Neural Information Processing Systems},
  34:20851--20862, 2021.

\bibitem[NR21]{naveh2021self}
Gadi Naveh and Zohar Ringel.
\newblock A self consistent theory of gaussian processes captures feature
  learning effects in finite cnns.
\newblock {\em Advances in Neural Information Processing Systems}, 34, 2021.

\bibitem[PN21]{pham2021global}
Huy~Tuan Pham and Phan-Minh Nguyen.
\newblock Global convergence of three-layer neural networks in the mean field
  regime.
\newblock {\em ICLR}, 2021.

\bibitem[RBC{\etalchar{+}}21]{rae2021scaling}
Jack~W Rae, Sebastian Borgeaud, Trevor Cai, Katie Millican, Jordan Hoffmann,
  Francis Song, John Aslanides, Sarah Henderson, Roman Ring, Susannah Young,
  et~al.
\newblock Scaling language models: Methods, analysis \& insights from training
  gopher.
\newblock {\em arXiv preprint arXiv:2112.11446}, 2021.

\bibitem[RVE18]{rotskoff2018parameters}
Grant Rotskoff and Eric Vanden-Eijnden.
\newblock Parameters as interacting particles: long time convergence and
  asymptotic error scaling of neural networks.
\newblock {\em Advances in neural information processing systems}, 31, 2018.

\bibitem[RYH22]{roberts2022principles}
Daniel~A Roberts, Sho Yaida, and Boris Hanin.
\newblock {\em The Principles of Deep Learning Theory: An Effective Theory
  Approach to Understanding Neural Networks}.
\newblock Cambridge University Press, 2022.

\bibitem[SMG14]{saxe2013exact}
Andrew~M Saxe, James~L McClelland, and Surya Ganguli.
\newblock Exact solutions to the nonlinear dynamics of learning in deep linear
  neural networks.
\newblock {\em ICLR}, 2014.

\bibitem[SR21]{seroussi2021separation}
Inbar Seroussi and Zohar Ringel.
\newblock Separation of scales and a thermodynamic description of feature
  learning in some cnns.
\newblock {\em arXiv preprint arXiv:2112.15383}, 2021.

\bibitem[SS20]{sirignano2020mean}
Justin Sirignano and Konstantinos Spiliopoulos.
\newblock Mean field analysis of neural networks: A central limit theorem.
\newblock {\em Stochastic Processes and their Applications}, 130(3):1820--1852,
  2020.

\bibitem[SS21]{sirignano2021mean}
Justin Sirignano and Konstantinos Spiliopoulos.
\newblock Mean field analysis of deep neural networks.
\newblock {\em Mathematics of Operations Research}, 2021.

\bibitem[SSL22]{saxe2022neural}
Andrew Saxe, Shagun Sodhani, and Sam~Jay Lewallen.
\newblock The neural race reduction: dynamics of abstraction in gated networks.
\newblock In {\em International Conference on Machine Learning}, pages
  19287--19309. PMLR, 2022.

\bibitem[Yai20]{yaida2019non}
Sho Yaida.
\newblock Non-gaussian processes and neural networks at finite widths.
\newblock {\em MSML}, 2020.

\bibitem[YH21]{yang2021tensoriv}
Greg Yang and Edward~J Hu.
\newblock Tensor programs iv: Feature learning in infinite-width neural
  networks.
\newblock In {\em International Conference on Machine Learning}, pages
  11727--11737. PMLR, 2021.

\bibitem[YS21]{yu2020normalization}
Jiahui Yu and Konstantinos Spiliopoulos.
\newblock Normalization effects on shallow neural networks and related
  asymptotic expansions.
\newblock {\em Foundations of Data Science}, 3(2):151--200, 2021.

\bibitem[ZBH{\etalchar{+}}17]{zhang2016understanding}
Chiyuan Zhang, Samy Bengio, Moritz Hardt, Benjamin Recht, and Oriol Vinyals.
\newblock Understanding deep learning requires rethinking generalization.
\newblock In {\em International Conference on Learning Representations,
  (ICLR)}, 2017.

\bibitem[ZVP21]{zavatone2021exact}
Jacob Zavatone-Veth and Cengiz Pehlevan.
\newblock Exact marginal prior distributions of finite bayesian neural
  networks.
\newblock {\em Advances in Neural Information Processing Systems}, 34, 2021.

\bibitem[ZVTP22]{zavatone2022contrasting}
Jacob~A. Zavatone-Veth, William~L. Tong, and Cengiz Pehlevan.
\newblock Contrasting random and learned features in deep bayesian linear
  regression.
\newblock {\em Phys. Rev. E}, 105:064118, 2022.

\end{thebibliography}
\bibliographystyle{alpha}
\end{document}